\documentclass[11pt]{article}
\usepackage[utf8]{inputenc}
\usepackage[T1]{fontenc}
\usepackage{geometry}
\usepackage{setspace}
\usepackage{mathtools}
\usepackage{graphicx, epstopdf}
\usepackage[arrow, matrix]{xy}
\usepackage{xcolor}
\usepackage{cite}
\usepackage{float}
\usepackage{algorithmic}
\usepackage{algorithm}
\usepackage{amssymb, amsmath, amsthm}
\usepackage{wrapfig}

\newtheorem{thm}{Theorem}
\newtheorem{defn}{Definition}
\newtheorem{lemma}{Lemma}
\newtheorem{pro}{Proposition}
\newtheorem{rk}{Remark}
\newtheorem{rt}{Result}

\def\s{\sigma}

\def\w{\omega}

\def\l{\lambda}

\def\b{\beta}

\title{Annealed Gibbs Measures with Data-Driven Interaction Potentials on Hierarchical Structures}

 \date{}
\vspace{-1cm}
\begin{document}
	
		\maketitle
		
		\begin{center}
			\textbf{L.U. Abdullaev}\\
		National University of Singapore, Singapore \\
				Email: \texttt{laziz.abdullaev@u.nus.edu}
			
			\textbf{F. Herrera} \\
			Department of Computer Science and Artificial Intelligence, University of Granada, E-18071 Granada, Spain \\
			Email: \texttt{herrera@decsai.ugr.es}
			
			\textbf{U. A. Rozikov} \\
			V.I. Romanovskiy Institute of Mathematics, Uzbekistan Academy of Sciences, 9, Universitet str., 100174, Tashkent, Uzbekistan; \\
			National University of Uzbekistan, 4, Universitet str., 100174, Tashkent, Uzbekistan; \\
			Institute for Advanced Study in Mathematics, Harbin Institute of Technology, \\Harbin 150001, China. \\
			Email: \texttt{rozikovu@yandex.ru}\\
			
			\textbf{M. V. Velasco} \\
			Departamento de Análisis Matemático, Facultad de Ciencias, Universidad de Granada, 18071 Granada, Spain \\
			Email: \texttt{vvelasco@ugr.es}
		\end{center}

\begingroup
	\begin{abstract}
We construct annealed Gibbs measures on continuous interaction parameters
and binary spins on Cayley trees of order $k\geq 1$, with coupling functions determined by
an empirical loss. The interaction parameters and spins are distinct
random variables: the former determine edge couplings, while the latter
represent binary labels. The training data specify the coupling function
and the joint Gibbs specification, rather than automatically selecting
a unique Gibbs measure on interaction parameters.

We derive necessary and sufficient compatibility conditions for the
finite-volume joint densities. These conditions take the form of nonlinear
integral equations for boundary laws and characterize the splitting
annealed Gibbs measures constructed here. For $k=1$,
the translation-invariant equations reduce to positive compact integral
operators. Under the stated hypotheses, they have a unique normalized
positive solution. For $k=2$, an explicit empirical-loss kernel reduces the
translation-invariant equations to polynomial equations. In this case, we show that hierarchical learning systems may exhibit phase-transition phenomena: for certain empirical kernels on Cayley trees order two, multiple Gibbs measures on spins emerge beyond a critical inverse temperature, corresponding to distinct equilibrium prediction regimes. 

For a selected annealed Gibbs measure, we distinguish its marginal on
interaction parameters from its conditional Gibbs measures on spins
at fixed interaction parameters. Spin prediction conditions on observed
labels and averages over the remaining spins and interaction parameters.
Numerical experiments with a non-separable data-dependent kernel
illustrate changes in the boundary-law solution set. We distinguish
selection of a measure for prediction from mathematical uniqueness,
and the numerical counts of symmetry classes from counts of individual
annealed Gibbs measures.

 Our results provide a new perspective on energy-based learning, where data do not merely determine an optimal model through minimization but define an entire probabilistic landscape of possible inference states.
\end{abstract}
\endgroup
	
		\noindent \textbf{Mathematics Subject Classifications (2010):} 82B20, 62C10, 68T07, 60J10.\\
		\noindent \textbf{Keywords:} Configuration, Ising model, empirical loss function, annealed Gibbs measures, Cayley tree, machine learning.

\section{Introduction}

Machine learning commonly starts with an empirical loss and selects a
predictor by minimizing that loss. Probabilistic approaches instead
retain a distribution over unknown quantities, making uncertainty and
the dependence on model selection explicit. Energy-based models provide
one connection between this viewpoint and statistical mechanics
\cite{AR,Ari,Bah,Beh,Krza,HRV,Le}.

Gibbs measures describe probability distributions associated with
specified Hamiltonians and boundary conditions \cite{FV,Ge}. In the
present paper, two kinds of variables must be distinguished throughout:
continuous \emph{interaction parameters} $\sigma(x)\in[0,1]$ and
binary \emph{spins} $\varphi(x)\in\{-1,+1\}$. A configuration of
interaction parameters determines the edge couplings between spins.
The spins can subsequently be used to represent binary labels.

Fix a dataset $\mathcal D$, a predictor $\mathcal F$, a loss $\ell$,
an inverse temperature $\beta$, and a Cayley-tree order $k$.
The empirical loss defines a deterministic coupling function
$\xi_N(t,u)=\mathcal L_N(t,u)$. For a parameter configuration
$\sigma$, the realized interaction parameter on an edge is
$J_{\langle x,y\rangle}(\sigma)=\xi_N(\sigma(x),\sigma(y))$.
Our Hamiltonian couples these interaction parameters to the spins;
it is not the empirical loss alone. In particular, the signs of the
spins affect the Boltzmann weight, so the construction is not simply
a Gibbs posterior that assigns larger probability to smaller losses.

There are three different measure-theoretic questions.

\begin{enumerate}
\item A \emph{Gibbs measure on interaction parameters} is a measure on
$\Omega=[0,1]^V$ for a separately specified parameter-only Hamiltonian.
It requires its own Gibbs specification. We discuss a parameter-only
transfer kernel as a comparison, but do not identify it with the
parameter marginal of our coupled model.
\item For a fixed $\sigma\in\Omega$, a \emph{Gibbs measure on spins}
is a measure on $\Sigma=\{-1,+1\}^V$ for the Ising Hamiltonian with
the fixed edge couplings $J_{\langle x,y\rangle}(\sigma)$. This is
the conditional, or quenched, spin problem.
\item An \emph{annealed Gibbs measure on interaction parameters and
spins} is a joint measure on $X=\Omega\times\Sigma$: the interaction
parameters are integrated over and the spins are summed over in the
same finite-volume distributions. These joint annealed Gibbs measures
are the principal objects constructed and analyzed in this paper.
\end{enumerate}
After this distinction has been made, we use ``joint annealed Gibbs
measure'' as shorthand for the third object. 

The data determine the coupling function, but this does not imply that
they determine a unique joint annealed Gibbs measure or a unique Gibbs
measure on interaction parameters. Each compatible normalized boundary
law determines one joint annealed Gibbs measure through its finite-volume
marginals. Different compatible boundary laws may produce different
joint measures for the same dataset and coupling function. The uniqueness
of the extension of one compatible family is therefore different from
uniqueness among all admissible boundary laws.

For prediction, an application must select a translation-invariant
joint annealed Gibbs measure $\mu^\star$ whenever several are available.
The word ``best'' refers to a specified learning criterion, for example
prediction error on held-out observations, together with a rule for
resolving ties. The consistency equations themselves do not supply such
a ranking or prove that its optimum is unique.

A selected $\mu^\star$ induces a distribution
$\nu^\star$ on interaction parameters and conditional distributions
$\rho^\star_\sigma$ on spins:
\[
 \nu^\star(B)=\mu^\star(B\times\Sigma),\qquad
 \mu^\star(d\sigma,d\varphi)
 =\nu^\star(d\sigma)\rho^\star_\sigma(d\varphi).
\]
The conditional spin distributions are defined for
$\nu^\star$-almost every $\sigma$ and satisfy the spin DLR equations
with the corresponding fixed couplings. The parameter marginal
$\nu^\star$ is not automatically the Gibbs measure of a parameter-only
Hamiltonian obtained by substituting the empirical loss. Summing over
spins changes the effective parameter weights.

This explains both the two-stage interpretation and its limits. After
selecting a joint measure, one can sample interaction parameters from
its marginal and examine the conditional Gibbs measure on spins for
typical realizations. Whether the fixed-coupling spin specification has
one or several Gibbs measures is a separate question. Our boundary-law analysis establishes uniqueness and multiplicity results for the joint parameter–spin model. These results concern the annealed system; an almost-sure classification of the corresponding quenched spin measures is a separate question. Likewise, specifying a Gibbs measure for the parameters and then selecting conditional spin measures produces a Gibbs measure for the joint Hamiltonian only if the resulting mixture satisfies its annealed DLR equations.

Cayley trees permit exact recursive compatibility equations
\cite{EE,Ro,Rbp}. Their role in this paper is an inference structure,
not an assertion that the geometry of the observed data is a tree.
We derive the joint boundary-law equations and study their
translation-invariant solutions by positive integral operators.
For $k=1$, the underlying graph is $\mathbb Z$; ``one-dimensional''
in that result refers to the graph, not merely to the fact that each
local interaction parameter is a scalar. For $k=2$, the equations are
nonlinear, and an explicit example reduces to polynomial equations
for boundary-law coordinates. The reported change from two algebraic
branches to three requires a distinction between algebraic candidates,
spin-reversal classes, and individual joint annealed Gibbs measures.

Our contributions are the construction of data-dependent joint annealed
Gibbs measures, their compatibility equations, the analysis of normalized
translation-invariant boundary laws, and conditional spin prediction
under a selected joint measure. Numerical experiments illustrate
additional solution branches for a non-separable kernel, while keeping
numerical observations distinct from the algebraic analysis of the
explicit example.

Section~2 defines the interaction parameters, spins, and joint
specification. Section~3 derives compatibility of the finite-volume
joint densities. Section~4 analyzes translation-invariant boundary laws
for empirical-loss coupling functions. Section~5 formulates conditional
spin prediction, and Section~6 presents numerical experiments and the
resulting selection procedure.

      \section{Preliminaries}

In this section, we recall the main definitions and introduce the necessary notation that will be used throughout the paper.
      
   \subsection{Definitions}  {\bf Cayley tree.} A Cayley tree \(\Gamma^k=(V,L)\), where \(V\) denotes the set of vertices and
\(L\) the set of edges, with branching factor \(k\geq 1\), is a connected
infinite graph in which every vertex has exactly \(k+1\) nearest neighbors.
The graph \(\Gamma^k\) is acyclic, i.e., it contains no closed paths or
cycles. 

Although the tree itself has no boundary due to its infinite
structure, finite subtrees (truncations) of the Cayley tree are frequently
considered in physical and computational applications.

Fix a vertex \(x^0\in V\), which will be referred to as the \emph{root} of the tree. 
A vertex \(y\in V\) is called a \emph{direct successor} of a vertex \(x\in V\) if \(x\) is 
the predecessor of \(y\) on the unique path connecting the root \(x^0\) to \(y\). 
Equivalently, this means that
\[
d(x^0,y)=d(x^0,x)+1,\qquad d(x,y)=1,
\]
where \(d(x,y)\) denotes the graph distance, i.e., the number of edges in the shortest 
path connecting \(x\) and \(y\). The set of all direct successors of a vertex \(x\in V\) 
is denoted by \(S(x)\).

For a fixed root \(x^0\in V\), the \emph{\(n\)-th level of the tree} is defined by
\[
W_n:=W_n(x^0)=\{x\in V:\ d(x,x^0)=n\},
\]
and the ball of radius \(n\) centered at \(x^0\) is given by
\begin{equation}\label{lp*}
V_n=\{x\in V:\ d(x,x^0)\leq n\},\qquad
L_n=\{l=\langle x,y\rangle\in L:\ x,y\in V_n\}.
\end{equation}

Since the root \(x^0\) is fixed throughout the paper, we omit it from the notation 
whenever there is no risk of ambiguity.

For any vertex \(x\in W_n\), the set of its \emph{direct successors} is given by

\begin{equation}\label{sx}
S(x)=\{y\in W_{n+1}: \langle x,y\rangle\in L\}.
\end{equation}
Cayley trees play an important
role in statistical mechanics as hierarchical lattices for studying phase
transitions and joint annealed Gibbs measures. They also appear in various areas such as
network theory, ecology, and computer science, where they provide natural
models for branching processes and hierarchical structures.

{\bf Interaction parameters and spins.}
At each vertex of the tree, spins are assigned values in $\{-1,+1\}$,
and interaction parameters are assigned values in $[0,1]$. For
$A\subseteq V$, the two configuration spaces are
\[
 \Omega_A=[0,1]^A,\qquad \Sigma_A=\{-1,+1\}^A,
 \qquad \sigma_A\in\Omega_A,\quad\varphi_A\in\Sigma_A.
\]
Thus $\sigma_A:A\to[0,1]$ is a configuration of interaction parameters,
whereas $\varphi_A:A\to\{-1,+1\}$ is a spin configuration.
Write $\Omega=[0,1]^V$, $\Sigma=\{-1,+1\}^V$, and
$X=\Omega\times\Sigma$, equipped with its product Borel sigma-algebra.

The formal joint Hamiltonian is
\
\begin{equation}\label{uy1}
 H(\sigma,\varphi)
 =-\sum_{\langle x,y\rangle\in L}
 \xi_{\sigma(x)\sigma(y)}\varphi(x)\varphi(y),
\end{equation}
where $\xi:[0,1]^2\to\mathbb R$ is a prescribed bounded measurable
coupling function. The quantity
$J_{\langle x,y\rangle}(\sigma)=\xi_{\sigma(x)\sigma(y)}$
is the realized interaction parameter on the edge $\langle x,y\rangle$.
The function $\xi$ is deterministic once the data-dependent model is
specified; randomness in the edge interaction parameters comes from
the random configuration $\sigma$. If $\xi_{tu}\equiv J$, the spin
part of \eqref{uy1} is the usual Ising Hamiltonian.

For each vertex $x$, let $\lambda_x$ be a copy of Lebesgue measure
on $[0,1]$. For a finite set $A$, use the product measure
\[
 \lambda_A=\bigotimes_{x\in A}\lambda_x.
\]
The reference measure on $\Omega_A\times\Sigma_A$ is
$\lambda_A$ times counting measure on $\Sigma_A$. Consequently,
expressions such as $\mu^{(n)}(\sigma_n,\varphi_n)$ below are
\emph{joint densities}: probabilities are obtained by integrating
over interaction parameters and summing over spins. 

For $\omega=(o,p)\in X$, let
\[
 H_n(\sigma_n,\varphi_n\mid\omega)
 =-\sum_{\langle x,y\rangle\in L_n}
 \xi_{\sigma_n(x)\sigma_n(y)}\varphi_n(x)\varphi_n(y)
 -\sum_{\substack{\langle x,y\rangle\in L\\x\in V_n,\ y\in W_{n+1}}}
 \xi_{\sigma_n(x)o(y)}\varphi_n(x)p(y).
\]

Here $V_n$, $W_n$, and $L_n$ are defined in \eqref{lp*} and
\eqref{sx}; the exterior is fixed in this conditional energy.
The corresponding joint Gibbs density is
\[
 \nu^{V_n}_{\omega|_{W_{n+1}}}(\sigma_n,\varphi_n)
 =\frac{\exp(-\beta H_n(\sigma_n,\varphi_n\mid\omega))}
 {Z_n(\omega|_{W_{n+1}})},\qquad \beta=1/T,
\]
where
\[
 Z_n(\omega|_{W_{n+1}})
 =\sum_{\varphi_n\in\Sigma_{V_n}}
 \int_{\Omega_{V_n}}
 \exp(-\beta H_n(\sigma_n,\varphi_n\mid\omega))
 \lambda_{V_n}(d\sigma_n).
\]
This normalizer is finite and strictly positive.
For a Borel event $C\subseteq X$, the finite-volume joint kernel is
\[
 \gamma_{V_n}(C\mid\omega)
 =\sum_{\varphi_n\in\Sigma_{V_n}}\int_{\Omega_{V_n}}
 \mathbf1_C\bigl((\sigma_n,\varphi_n)\omega_{V_n^c}\bigr)
 \nu^{V_n}_{\omega|_{W_{n+1}}}(\sigma_n,\varphi_n)
 \lambda_{V_n}(d\sigma_n).
\]
A probability measure $\mu$ on $X$ is an \emph{annealed Gibbs measure
on interaction parameters and spins} if
\[
 \mu(C\mid\mathcal B_{V_n^c})(\omega)
 =\gamma_{V_n}(C\mid\omega)
 \quad\text{for $\mu$-almost every $\omega$},
\]
for every $n$ and Borel $C$, where $\mathcal B_{V_n^c}$ is generated
by both kinds of exterior coordinates. In particular,
$\mu(C)=\int_X\gamma_{V_n}(C\mid\omega)\mu(d\omega)$.

We study the joint annealed Gibbs measures generated by compatible
boundary laws, with particular attention to translation-invariant
boundary laws. The definitions distinguish this joint problem from
a Gibbs measure on interaction parameters alone and from a Gibbs
measure on spins with the interaction parameters held fixed.

\subsection{Data-dependent couplings and selection for prediction}\label{mot}
Let $\mathcal D=\{(x_i,y_i)\}_{i=1}^N$ be a training dataset,
with inputs $x_i\in\mathcal X$ and targets $y_i\in\mathbb R$.
For a predictor $\mathcal F:\mathcal X\times[0,1]^2\to\mathbb R$
and a loss $\ell:\mathbb R\times\mathbb R\to[0,\infty)$, define
\[
 \mathcal L_N(t,u)
 =\frac1N\sum_{i=1}^N\ell(y_i,\mathcal F(x_i;t,u)),
 \qquad t,u\in[0,1].
\]
The arguments $t,u$ represent interaction parameters at neighboring
vertices. We set $\xi_N(t,u)=\mathcal L_N(t,u)$ and use this function
as the coupling in \eqref{uy1}. Once the dataset, predictor, and loss
are fixed, this coupling function is fixed. The configurations
$\sigma$ and $\varphi$ nevertheless remain random under an annealed
Gibbs measure on interaction parameters and spins.

For aligned spins, the edge Boltzmann factor is
$\exp(+\beta\mathcal L_N(t,u))$, while for opposite spins it is
$\exp(-\beta\mathcal L_N(t,u))$. Thus the empirical loss specifies
the spin coupling, rather than serving as the complete energy of a
parameter-only empirical-risk posterior. This interpretation is
essential when relating the model to a learning objective.
\begin{rk}
Data-dependent coupling functions specify a family of annealed Gibbs
measures on interaction parameters and spins; they do not automatically
choose one member of that family. For a prediction task, the application
selects an admissible translation-invariant joint annealed Gibbs measure.
A declared criterion, such as held-out prediction error, can be used to
choose the ``best'' candidate, with a tie rule when necessary. No
optimality or uniqueness assertion for this selection is proved here.

After a choice $\mu^\star$ has been made, its parameter marginal
$\nu^\star$ is uniquely determined by that choice. This uniqueness
must not be confused with a claim that the dataset selects a unique
Gibbs measure on interaction parameters. Nor does the marginal
automatically have the parameter-only Gibbs specification with edge
factor $e^{-\beta\mathcal L_N}$.
\end{rk}
\begin{rk}
For a selected joint annealed Gibbs measure $\mu$, standard
disintegration gives a parameter marginal $\nu$ and conditional spin
measures $\rho_\sigma$ with
\[
 \mu(B\times A)=\int_B\rho_\sigma(A)\nu(d\sigma),
 \qquad B\subseteq\Omega,\quad A\subseteq\Sigma\text{ Borel}.
\]
For $\nu$-almost every parameter configuration, $\rho_\sigma$ is a
Gibbs measure on spins for the fixed couplings
$\xi_N(\sigma(x),\sigma(y))$. Other Gibbs measures on spins may
exist for those same couplings. Our results about joint boundary laws
do not supply an almost-sure uniqueness or non-uniqueness theorem for
that separate spin specification.

Different joint annealed Gibbs measures can also have the same
parameter marginal. For example, global spin reversal changes
$\varphi$ while leaving $\sigma$ unchanged. A pair of distinct
spin-reversed joint measures therefore has the same marginal on
interaction parameters. Counting joint measures alone cannot determine
the number of parameter marginals.
\end{rk}
\begin{rk}\label{rk:observations}
The input samples belong to $\mathcal X$ and are not initially
tree vertices. To use the tree for binary-label inference, assign the
samples to specified vertices $v_1,\ldots,v_N\in V$ and assume
$y_i\in\{-1,+1\}$. The spin observation event is
$
 \mathcal E=\{\varphi(v_i)=y_i:\ i=1,\ldots,N\}.$
For an unobserved vertex $x$ and a selected joint annealed Gibbs
measure $\mu$, define
\[
 \widehat\varphi_\mu(x)=\mathbb E_\mu[\varphi(x)\mid\mathcal E].
\]
This prediction averages over interaction parameters and unobserved
spins. It differs from prediction under a Gibbs measure on spins
with a fixed parameter configuration. General real-valued targets
can be used to form the empirical loss, but cannot be conditioned
on as binary spin values without an additional encoding.
\end{rk}
\begin{rk}\label{rk:tree-role}
The Cayley tree is an inference structure, rather than a model of
the geometry of the dataset. Its recursive organization makes the
joint compatibility equations exact and connects the boundary-law
description with message passing \cite{Ro,Rbp}. Fixed data embeddings
used to form the empirical loss are observed covariates.
\end{rk}
The remaining analysis concerns the joint annealed Gibbs measures
generated by these compatibility equations. 

\section{An integral equation}\label{sec:integraleq}

Recall that we write $x<y$ if the path from $x^0$ to $y$ goes through $x$. Call vertex $y$ a \emph{direct successor} of
$x$ if $y>x$ and $x,y$ are nearest neighbors. Denote by $S(x)$ the set of direct successors of $x$.
Observe that any vertex $x\ne x^0$ has $k$ direct successors and $x^0$ has $k+1$.

Let $h:\;x\in V\mapsto h_x=(h_{t,u,x}, (t,u)\in [0,1]\times \{-1,1\}) \in \mathbb R^{[0,1]\times \{-1,1\}}$ be mapping of $x\in V\setminus
\{x^0\}$ with $|h_{t,u,x}|<C$ where $C$ is a constant which does not depend on $t,u$.  Given
$n=1,2,\ldots$, consider the probability distribution $\mu^{(n)}$ on $\Omega_{V_n}\times  \Sigma_{V_{n}}$ defined by
\begin{equation}\label{uy2}
	\mu^{(n)}(\sigma_n,\varphi_n)=Z_n^{-1}\exp\left(-\beta H(\sigma_n, \varphi_n)
	+\sum_{x\in W_n}h_{\sigma(x),\varphi(x), x}\right).
\end{equation}

Here $\mu^{(n)}(\sigma_n,\varphi_n)$ denotes a joint density
with respect to $\lambda_{V_n}$ and counting measure on $\Sigma_{V_n}$.
Thus

\[
 Z_n=\sum_{\varphi_n\in\Sigma_{V_n}}\int_{\Omega_{V_n}}
 \exp\left(-\beta H(\sigma_n,\varphi_n)
 +\sum_{x\in W_n}h_{\sigma_n(x),\varphi_n(x),x}\right)
 \lambda_{V_n}(d\sigma_n).
\]

The functions $h_{t,s,x}$ are boundary functions indexed by an
interaction parameter $t$ and a spin $s$.

The probability distributions $\mu^{(n)}$ are compatible if for any $n\geq 1$ and configurations
$\sigma_{n-1}\in\Omega_{V_{n-1}}$, $\varphi_{n-1}\in \Sigma_{V_{n-1}}$ (see \cite{FV} and \cite{Ro} for several kind of compatibility):
\begin{equation}\label{uy4}
	\sum_{p_n\in \Sigma_{W_n}}\int_{\Omega_{W_n}}\mu^{(n)}(\sigma_{n-1}\vee o_n, \varphi_{n-1}\vee p_n)\lambda_{W_n}(d(o_n))=
	\mu^{(n-1)}(\sigma_{n-1}, \varphi_{n-1}).
\end{equation}
Here
$\sigma_{n-1}\vee o_n\in\Omega_{V_n}$ is the concatenation of
$\sigma_{n-1}$ and $o_n$.

For each specified compatible family of joint densities, there
exists a unique probability measure $\mu$ on $X$ whose marginals satisfy

\begin{equation}\label{muu}
 \mu\{\sigma|_{V_n}\in B,\ \varphi|_{V_n}=\varphi_n\}
 =\int_B\mu^{(n)}(\sigma_n,\varphi_n)
 \lambda_{V_n}(d\sigma_n),\qquad B\subseteq\Omega_{V_n}\text{ Borel}.
\end{equation}

This is uniqueness of the extension of the chosen compatible family,
not uniqueness among all joint annealed Gibbs measures for the same
coupling function.

\begin{defn}\label{uyd1} 
The measure \(\mu\) defined by \eqref{muu} is called a 
\emph{splitting annealed Gibbs measure on interaction parameters and spins} corresponding to the Hamiltonian (\ref{uy1}) with the 
boundary functions \(h=\{h_x:x\in V\setminus\{x^0\}\}\) if the family of finite-volume 
distributions \(\{\mu^{(n)}\}_{n\geq1}\) given by \eqref{uy2} is compatible in the 
sense of \eqref{uy4}. In other words, \(\mu\) is the probability measure on 
\(X\) whose finite-volume marginals coincide with the Gibbs distributions 
\(\mu^{(n)}\) determined by the Hamiltonian and the boundary functions \(h_x\).
\end{defn}

\begin{pro}\label{uyp1} The probability distributions
	$	\mu^{(n)}(\sigma_n,\varphi_n)$, $n=1,2,\ldots$, in (\ref{uy2}) are compatible iff for
	any $x\in V\setminus\{x^0\}$
	the following equation holds:
	
	$$
	f(t,x)=\prod_{y\in S(x)}{\int_0^1\eta_{tu}f(u,y)du +\int_0^1\eta^{-1}_{tu}g(u,y)du\over \int_0^1\eta^{-1}_{0u}f(u,y)du+\int_0^1\eta_{0u}g(u,y)du}.
	$$
	\begin{equation}\label{uy5}
		g(t,x)=\prod_{y\in S(x)}{\int_0^1\eta^{-1}_{tu}f(u,y)du +\int_0^1\eta_{tu}g(u,y)du\over \int_0^1\eta^{-1}_{0u}f(u,y)du+\int_0^1\eta_{0u}g(u,y)du}.
	\end{equation}
	Here
	$$f(t, x)=\exp(h_{t,+,x}-h_{0, -, x}), \ \ g(t, x)=\exp(h_{t,-,x}-h_{0, -, x}), \ \ \eta_{tu}=\exp(\beta \xi_{tu}).$$ and
	$du=\l(du)$ is the Lebesgue measure.
\end{pro}
\begin{proof} {\sl Necessity.}  Suppose that (\ref{uy4}) holds; we shall prove (\ref{uy5}).
	Substituting (\ref{uy2}) into (\ref{uy4}), obtain that  for any
	configurations $\sigma_{n-1}$: $x\in V_{n-1}\mapsto\sigma_{n-1}(x)\in [0,1]$, $\varphi_{n-1}$: $x\in V_{n-1}\mapsto\varphi_{n-1}(x)\in \Sigma$:
	
	\begin{equation}\label{uy6}\frac{Z_{n-1}}{Z_n}	\sum_{p_n\in \Sigma_{W_n}}\int_{\Omega_{W_n}}
		\exp\left(\sum_{x\in W_{n-1}}\sum_{y\in S(x)} (\beta \xi_{\sigma_{n-1}(x)\omega_n(y)}\varphi_{n-1}(x)p_n(y)+
		h_{\omega_n(y),p_n(y),y})\right)\l_{W_n}(d\w_n)
	\end{equation}
	$$ =\exp\left(\sum_{x\in
		W_{n-1}}h_{\sigma_{n-1}(x), \varphi_{n-1}(x),x}\right),$$
	
	From (\ref{uy6}) we get:
	$${Z_{n-1}\over Z_n}\sum_{p_n\in \Sigma_{W_n}}\int_{\Omega_{W_n}}
	\prod_{x\in W_{n-1}}\prod_{y\in S(x)} \exp\,(\beta \xi_{\sigma_{n-1}(x)\omega_n(y)}\varphi_{n-1}(x)p_n(y)+
	h_{\omega_n(y),p_n(y),y})d(\w_n(y))$$
	\begin{equation}\label{fh} = \prod_{x\in W_{n-1}} \exp\,(h_{\sigma_{n-1}(x), \varphi_{n-1}(x), x}). \end{equation}
	
	Fix $x\in W_{n-1}$ and consider two configurations $\varphi_{n-1}=\overline{\varphi}_{n-1}$ and
	$\varphi_{n-1}=\tilde{\varphi}_{n-1}$ on $W_{n-1}$ which coincide on $W_{n-1}\setminus \{x\}$,
	and rewrite now the equality (\ref{fh}) for $\overline{\varphi}_{n-1}(x)=1$ and $\tilde{\varphi}_{n-1}(x)=-1$. Then dividing both of them to the equality corresponding to $\tilde{\varphi}_{n-1}(x)=-1$ and $t=0$
	we get (\ref{uy5}).

	{\sl Sufficiency.} Suppose that (\ref{uy5}) holds. It is equivalent to
	the representations
	\begin{equation}\label{uy7}
		\prod_{y\in S(x)}\sum_{\epsilon\in \{-1,1\}}\int_0^1\exp\,(\beta \xi_{tu}s\epsilon +
		h_{u,\epsilon, y})du= a(x)\exp\,(h_{t,s,x}), t\in [0,1], s=-1,+1
	\end{equation}
	for some function $a(x)>0, x\in V.$
	We have
	$${\rm LHS \ \ of \ \  (\ref{uy4})}=\frac{1}{Z_n}\exp(-\b H(\s_{n-1}, \varphi_{n-1}))
	\l_{V_{n-1}}(d(\s_n))\times
	$$
	\begin{equation}\label{uy8}
		\prod_{x\in W_{n-1}} \prod_{y\in S(x)}\sum_{\varepsilon\in \{-1,1\}}\int_0^1\exp\,(\beta \xi_{\s_{n-1}(x)u}\varphi_{n-1}(x)\varepsilon+
		h_{u,\varepsilon,y})du,
	\end{equation}
	(where LHS means left hand side). 
	Substituting (\ref{uy7}) into (\ref{uy8}) and denoting $A_n(x)=\prod_{x \in W_{n-1}} a(x)$,
	we get
	\begin{equation}\label{uy9}
		{\rm RHS\ \  of\ \  (\ref{uy8}) }=
		\frac{A_{n-1}}{Z_n}\exp(-\b H(\s_{n-1},\varphi_{n-1})) \lambda_{V_{n-1}}(d\s)\prod_{x\in W_{n-1}}
		h_{\s_{n-1}(x),\varphi_{n-1}(x),x}.
	\end{equation}
	
	Since $\mu^{(n)}$, $n \geq 1$ is a probability, we should have
	$$ \sum_{\varphi_{n-1}\in \Sigma_{V_{n-1}}}\int_{\Omega_{V_{n-1}}}\lambda_{V_{n-1}}(d\s_{n-1})
	\sum_{p_n\in \Sigma_{W_n}}\int_{\Omega_{W_n}} \lambda_{W_{n}}(d\omega_{n})
	\mu^{(n)}(\sigma_{n-1}\vee \omega_n, \varphi_{n-1}\vee p_n) = 1. $$
	
	Hence from (\ref{uy9})  we get $Z_{n-1}A_{n-1}=Z_n$, and (\ref{uy4}) holds.
\end{proof}

Proposition~\ref{uyp1} identifies compatible boundary functions
with the joint finite-volume construction. Each compatible
boundary law specifies one splitting annealed Gibbs measure on
interaction parameters and spins. 

\section{Hamiltonians with empirical-loss coupling functions}
	
Let \(\mathcal X\) be a prescribed input space and let
$
\{(x_i,y_i)\}_{i=1}^{N}\subset \mathcal X\times\mathbb R
$
be a given training dataset, where \(x_i\in\mathcal X\) denotes an input sample and \(y_i\in\mathbb R\) is the corresponding target value.

Assume that a prediction mechanism is specified by a given function
$
\mathcal F:\mathcal X\times [0,1]^2\to\mathbb R.
$

For every input \(x\in\mathcal X\) and every pair of local parameters
\((t,u)\in[0,1]^2\), the quantity
$
\mathcal F(x;t,u)
$
represents the prediction generated by two neighboring vertices carrying parameters \(t\) and \(u\).

Let
$
\ell:\mathbb R\times\mathbb R\to[0,\infty)
$
be a prescribed loss function that measures the discrepancy between a target value \(y\) and a prediction \(\hat y\). Typical examples include the squared loss
$
\ell(y,\hat y)=(y-\hat y)^2
$
and the absolute loss
$
\ell(y,\hat y)=|y-\hat y|.
$

For a fixed pair \((t,u)\), the empirical interaction loss associated with the dataset is defined by
\[
\mathcal L_N(t,u)
=
\frac1N
\sum_{i=1}^{N}
\ell\!\left( y_i,\mathcal{F}(x_i;t,u)\right)
\]

Thus, \(\mathcal L_N(t,u)\) represents the average prediction error produced by the neighboring parameters \(t\) and \(u\) over the training sample.
Set
	\begin{equation}\label{eta}
	\eta_{tu}=\exp\big(\beta \mathcal L_N(t,u)\big).
	\end{equation}
%
%
	
	To adapt this framework to the Cayley tree setting, we proceed as follows:
	\begin{itemize}
		\item Associate to each vertex \(x\in V\) a local parameter \(\theta_x \in [0,1]\);
		\item Identify \(\sigma(x)\equiv \theta_x\);
		\item Retain spin variables \(\varphi(x)\in\{-1,1\}\).
	\end{itemize}
	
	We define a data-dependent interaction along edges by
	\[
	\xi_{\sigma(x)\sigma(y)} \;\longrightarrow\;
	\xi(\theta_x,\theta_y)
	:=
	\frac{1}{N}\sum_{i=1}^N
	\ell\big(y_i,\mathcal F(x_i;\theta_x,\theta_y)\big).
	\]
		We consider the joint annealed Hamiltonian \eqref{uy1} with
this empirical coupling function. Here $\sigma(x)=\theta_x$ is an
interaction parameter and $\varphi(x)$ is a spin; neither is a
boundary-law coefficient.
	
	\subsection{Analysis of the solutions of the system (\ref{uy5}) in the case (\ref{eta}).}  
	
    {\bf The case \(k=1\).}
	The Cayley tree of order \(k=1\) coincides with the one-dimensional integer lattice \(\mathbb Z\).
	The root \(0\) is fixed; hence it suffices to solve the system independently on each half of \(\mathbb Z\).
	Identifying vertices with \(n=0,1,2,\ldots\), we have
	$
	S(n)=\{n+1\}.
	$
	
	In this case the compatibility equations \eqref{uy5} become 
	\begin{equation}\label{k1-f}
		f_n(t)=
		\frac{\displaystyle \int_0^1 \eta_{tu}\,f_{n+1}(u)\,du
			+\int_0^1 \eta_{tu}^{-1}\,g_{n+1}(u)\,du}
		{\displaystyle \int_0^1 \eta_{0u}^{-1}\,f_{n+1}(u)\,du
			+\int_0^1 \eta_{0u}\,g_{n+1}(u)\,du},
	\end{equation}
	\begin{equation}\label{k1-g}
		g_n(t)=
		\frac{\displaystyle \int_0^1 \eta_{tu}^{-1}\,f_{n+1}(u)\,du
			+\int_0^1 \eta_{tu}\,g_{n+1}(u)\,du}
		{\displaystyle \int_0^1 \eta_{0u}^{-1}\,f_{n+1}(u)\,du
			+\int_0^1 \eta_{0u}\,g_{n+1}(u)\,du}.
	\end{equation}
	where $\eta_{tu}$ is given by (\ref{eta}).
	
	We first study translation-invariant solutions
	$
	f_n(t)\equiv f(t), \ \	g_n(t)\equiv g(t)$.
		
	Then the system is reduced to
	\begin{equation}\label{f1}
		f(t)=
		\frac{\displaystyle \int_0^1 \eta_{tu}f(u)\,du
			+\int_0^1 \eta_{tu}^{-1}g(u)\,du}
		{D},
	\end{equation}
	\begin{equation}\label{g1}
		g(t)=
		\frac{\displaystyle \int_0^1 \eta_{tu}^{-1}f(u)\,du
			+\int_0^1 \eta_{tu}g(u)\,du}
		{D},
	\end{equation}
	where
	\[
	D=
	\int_0^1 \eta_{0u}^{-1}f(u)\,du
	+
	\int_0^1 \eta_{0u}g(u)\,du.
	\]
    	
	Introduce 
	\[
	s(t)=f(t)+g(t),
	\qquad
	r(t)=f(t)-g(t).
	\]
	
	Adding and subtracting \eqref{f1} and \eqref{g1}, we obtain
	$$	s(t)	= \frac1D
		\int_0^1
		\bigl(\eta_{tu}+\eta_{tu}^{-1}\bigr)s(u)\,du, \ \ 
		r(t)
		=
		\frac1D
		\int_0^1
		\bigl(\eta_{tu}-\eta_{tu}^{-1}\bigr)r(u)\,du.$$
	
	Since \(\eta_{tu}=e^{\beta \mathcal L_N(t,u)}\), we have
	\[
	\eta_{tu}+\eta_{tu}^{-1}=2\cosh(\beta \mathcal L_N(t,u)),
	\qquad
	\eta_{tu}-\eta_{tu}^{-1}=2\sinh(\beta \mathcal L_N(t,u)).
	\]
	
	Thus,
	$$
		s(t)=\frac{2}{D}
		\int_0^1 \cosh(\beta \mathcal L_N(t,u))\,s(u)\,du, \ \
		r(t)=\frac{2}{D}
		\int_0^1 \sinh(\beta \mathcal L_N(t,u))\,r(u)\,du.
	$$
	
	Define the operators
	$$
	(K_c\phi)(t)=\int_0^1 \cosh(\beta \mathcal L_N(t,u))\phi(u)\,du,
	\ \
	(K_s\phi)(t)=\int_0^1 \sinh(\beta \mathcal L_N(t,u))\phi(u)\,du.
	$$
	
	Then
	\begin{equation}\label{eigen}
		K_c s=\lambda s, \ \ 
		K_s d=\lambda d, \ \ \lambda=D/2.
	\end{equation}

%
	
Assume that the empirical loss $\mathcal L_N(t,u)$ is continuous on $[0,1]^2$.
Then the kernels
\[
A(t,u)=\cosh\!\bigl(\beta \mathcal L_N(t,u)\bigr),
\qquad
B(t,u)=\sinh\!\bigl(\beta \mathcal L_N(t,u)\bigr)
\]
are continuous on $[0,1]^2$ as well.

Hence the associated integral operators
$
(K_c\phi)(t)$ and $(K_s\phi)(t)$
are compact operators on $C([0,1])$ and on $L^2([0,1])$.
Indeed, continuity of the kernels implies that $K_c$ and $K_s$ are Hilbert–Schmidt operators on $L^2([0,1])$, and therefore compact; compactness on $C([0,1])$ follows from the Arzelà–Ascoli theorem applied to the image of the unit ball.

%

Since
\[
A(t,u)=\cosh\!\bigl(\beta \mathcal L_N(t,u)\bigr) > 0
\quad \text{for all } (t,u)\in[0,1]^2,
\]
it follows that $K_c$ is a strongly positive operator, i.e.,
\[
\phi \ge 0,\ \phi \not\equiv 0 \quad \Longrightarrow \quad (K_c\phi)(t) > 0 \ \text{for all } t.
\]

Therefore, by the Krein–Rutman theorem \cite{Du}, the spectral radius
$r(K_c)$ is a simple eigenvalue of $K_c$, and there exists a unique (up to multiplication by a positive scalar) eigenfunction
$s_* \in C([0,1])$, strictly positive on $[0,1]$, such that
$
K_c s_* = r(K_c)\, s_*.$
Moreover, every nonnegative eigenfunction associated with $r(K_c)$ is proportional to $s_*$.

To make a positivity comparison between $K_c$ and $K_s$, assume now that
\[
\mathcal L_N(t,u) > 0 \quad \text{for all } (t,u)\in[0,1]^2.
\]
Then
\[
0 < B(t,u) = \sinh\!\bigl(\beta \mathcal L_N(t,u)\bigr)
< \cosh\!\bigl(\beta \mathcal L_N(t,u)\bigr) = A(t,u),
\]
so that
\[
0 < K_s \le K_c \quad \text{(pointwise ordering of kernels)}.
\]

Both operators are therefore positive, and $K_s$ is also strongly positive.

By the monotonicity property of the spectral radius for positive compact operators,
\[
0 < K_s \le K_c \ \text{and } K_s \ne K_c
\quad \Longrightarrow \quad
r(K_s) < r(K_c).
\]

\medskip

\noindent
\textbf{Consequence for the antisymmetric component.}
Since the normalization constant $\lambda$ is determined by the system \eqref{eigen}, consistency of both equations would require a common eigenvalue of $K_c$ and $K_s$.

However, because $r(K_s) < r(K_c)$, the only possible way for both equations to be compatible at the dominant spectral level is that the antisymmetric component vanishes. More precisely, $K_s$ cannot support a nontrivial positive or dominant eigenmode at the same spectral level as $K_c$.

Hence the only admissible solution in the translation-invariant class is
$
r(t) \equiv 0,
$ 
which implies $
f(t) = g(t).$
Thus the boundary law is spin-flip symmetric. It generates the unique joint annealed Gibbs measure within the positive translation-invariant boundary-law class considered here.

	\begin{thm}
		\label{thm:chain} Assume that $\mathcal L_N$ is continuous on $[0,1]^2$ and $\mathcal L_N(t,u)>0$ for all $(t,u)\in[0,1]^2$.
		Then the normalized compatibility equations \eqref{uy5} for joint annealed Gibbs measures on a Cayley tree of order $1$ admit a unique positive translation-invariant boundary-law solution, determined by the dataset through $\mathcal F$.
	\end{thm}
		\begin{proof}
		Any translation-invariant solution satisfies \(f(t)=g(t)=h(t)\), leading to the eigenvalue problem
		\[
		K_N h = r_N h,
		\quad
		K_N\phi=\int_0^1 \cosh(\beta \mathcal L_N(t,u))\phi(u)\,du.
		\]
		By Krein–Rutman theorem, a unique positive eigenfunction exists, giving the desired solution.
	\end{proof}

	\subsection{Explicit solutions for particular datasets}
	
	In general, the positive eigenfunction $\psi_N$, satisfying
	\begin{equation}\label{eigen-problem}
		\int_0^1
		\cosh\!\bigl(\beta \mathcal L_N(t,u)\bigr)
		\psi_N(u)\,du
		=
		r_N\,\psi_N(t),
	\end{equation}
	cannot generally be determined in closed form. The next three cases concern the tree of order $k=1$; the subsequent example treats $k=2$. Explicit solutions are available for several classes of empirical loss functions.
	
	\subsubsection{Case 1. Constant empirical loss}
	
	Suppose that $\mathcal L_N(t,u)\equiv c$.
	
	Then
	\[
	K_N\phi
	=
	\cosh(\beta c)\int_0^1\phi(u)\,du.
	\]
	Hence $K_N$ is a rank-one operator.
	
	The eigenvalue equation becomes
	\[
	\cosh(\beta c)\int_0^1\psi_N(u)\,du
	=
	r_N\psi_N(t).
	\]
	Therefore $\psi_N(t)$ must be constant. After normalization,
	$\psi_N(t)\equiv 1,
	$
	and
	$
	r_N=\cosh(\beta c).
	$
	Consequently, $f(t)=g(t)=1.$
		
\subsubsection{Case 2. Additive empirical loss}
Assume that the empirical loss has the additive structure $
\mathcal L_N(t,u)=a(t)+a(u),$ 
where \(a:[0,1]\to\mathbb R\) is a continuous function. Let
\begin{equation} \label{nuevo}
p(t)=e^{\beta a(t)}, \qquad q(t)=e^{-\beta a(t)}.
\end{equation}
Then the kernel factorizes as
$
\eta_{tu}=p(t)p(u)$,  $\eta_{tu}^{-1}=q(t)q(u).
$

The translation-invariant compatibility equations reduce to
\[
f(t)=
\frac{
	p(t)\int_0^1 p(u)f(u)\,du
	+
	q(t)\int_0^1 q(u)g(u)\,du
}{D},
\]
\[
g(t)=
\frac{
	q(t)\int_0^1 q(u)f(u)\,du
	+
	p(t)\int_0^1 p(u)g(u)\,du
}{D},
\]
where
\[
D=\int_0^1 q(u)f(u)\,du+\int_0^1 p(u)g(u)\,du.
\]

Introduce the scalar quantities
\[
A=\int_0^1 p(u)f(u)\,du, \quad
B=\int_0^1 q(u)g(u)\,du, \quad
C=\int_0^1 q(u)f(u)\,du, \quad
E=\int_0^1 p(u)g(u)\,du.
\]

\begin{rk}\label{rkp}
By Proposition 1, we have that \(f(t)\) and \(g(t)\) are strictly positive, while \(p(t)\) and \(q(t)\) are strictly positive by (20). Hence,
$
A,B,C,D,E>0.$

\end{rk}

With this notation, the solutions take the form
\begin{equation}\label{fA}
f(t)=\frac{A p(t)+B q(t)}{D}, \qquad
g(t)=\frac{C q(t)+E p(t)}{D}.
\end{equation}
Thus every solution of the compatibility equations belongs to the
two-dimensional space
\[
\operatorname{span}\{e^{\beta a(t)},\,e^{-\beta a(t)}\}.
\]

\medskip

{\bf Case $f=g$.} Now consider symmetric translation-invariant solutions: $
f=g=h$.

In this case,
\[
h(t)=\frac{A e^{\beta a(t)}+B e^{-\beta a(t)}}{A+B}
=\alpha e^{\beta a(t)}+(1-\alpha)e^{-\beta a(t)},
\]
where
\begin{equation}\label{alp}
	\alpha=\frac{A}{A+B}.
\end{equation}
By Remark \ref{rkp} we have that $\alpha\in (0,1).$

Substituting this representation into the definitions of \(A\) and \(B\),
we obtain
\[
A=\alpha I_2+(1-\alpha), \qquad
B=\alpha+(1-\alpha)I_{-2},
\]
where
\begin{equation}\label{I12}
I_2=\int_0^1 e^{2\beta a(u)}\,du, \qquad
I_{-2}=\int_0^1 e^{-2\beta a(u)}\,du.
\end{equation}

Using (\ref{alp}), we arrive at the consistency equation
$
\alpha(1-\alpha)(I_2-I_{-2})+1-2\alpha=0.
$

Let $\Delta:=I_2-I_{-2}$.

Then the equation can be rewritten as
$
\Delta\alpha^2-(\Delta-2)\alpha-1=0.$

If \(\Delta\neq0\), the solutions are
\[
\alpha_{\pm}
=
\frac{\Delta-2\pm\sqrt{\Delta^2+4}}{2\Delta}.
\]

A direct inspection shows that exactly one root lies in  \((0,1)\):
\[
\alpha_+\in(0,1), \quad \alpha_-<0 \quad \text{if } \Delta>0,
\]
\[
\alpha_-\in(0,1), \quad \alpha_+>1 \quad \text{if } \Delta<0.
\]

In the degenerate case \(\Delta=0\), the equation reduces to
$
\alpha=\frac12.
$


Define 
$$
	\alpha^*=\left\{\begin{array}{lll}
			\alpha_- \ \ \text{if} \ \ \Delta<0,\\[2mm]
			\frac12, \ \ \text{if} \ \ \Delta=0,\\[2mm]
\alpha_+, \ \ \text{if} \ \ \Delta>0.
\end{array}\right.
$$
Then \(\alpha^*\in (0,1)\), and the corresponding function
\begin{equation}\label{hu}
	h(t)
	=
	\alpha^* e^{\beta a(t)}
	+
	(1-\alpha^*)e^{-\beta a(t)}
\end{equation}
is strictly positive and generates a symmetric translation-invariant joint annealed Gibbs measure.

Thus the additive empirical loss yields a unique symmetric
translation-invariant joint annealed Gibbs measure determined by the solution
\(\alpha^*\) of the above quadratic equation.

{\bf Case $f\ne g$.}  We now consider the non-symmetric case \(f \neq g\). 
Introduce the representation
\begin{equation}\label{fx}
f(t)=x\,p(t)+y\,q(t), \qquad
g(t)=z\,q(t)+w\,p(t),
\end{equation}
for some positive constants \(x,y,z,w>0\). Substituting into the definitions of
\(A,B,C,E\), we obtain the linear relations
\begin{equation}\label{AB}
A = x I_2 + y, \qquad
B = z I_{-2} + w, \qquad
C = x + y I_{-2}, \qquad
E = z + w I_2,
\end{equation}
where $I_2$ and $I_{-2}$ are defined by (\ref{I12}).

On the other hand, comparing coefficients in (\ref{fA}) and (\ref{fx}) 
yields the algebraic consistency system
\begin{equation}\label{xy}
x = \frac{A}{D}, \qquad y = \frac{B}{D}, \qquad
z = \frac{C}{D}, \qquad w = \frac{E}{D}.
\end{equation}

Substituting (\ref{AB}) into (\ref{xy}) we obtain the system

\begin{equation}\label{sysxy}
	\begin{aligned}
		xD &= x I_2 + y,\\
		yD &= z I_{-2} + w,\\
		zD &= x + y I_{-2},\\
		wD &= z + w I_2.
	\end{aligned}
\end{equation}

Subtracting the fourth equation from the first yields

\begin{equation}\label{eqxw}
	(D-I_2)(x-w)=y-z.
\end{equation}

Similarly, subtracting the third equation from the second gives

\begin{equation}\label{eqyz}
	(D+I_{-2})(y-z)=-(x-w).
\end{equation}

Substituting (\ref{eqxw}) into (\ref{eqyz}), we obtain
$
\Bigl[(D+I_{-2})(D-I_2)+1\Bigr](x-w)=0.$

We show that the coefficient is strictly positive. From the first equation of
(\ref{sysxy}), we get
\[
x(D-I_2)=y.
\]
By Remark \ref{rkp} and (\ref{xy}) we have  \(x>0\) and \(y>0\), therefore, a positive solution exists if 
$
D-I_2>0.
$
Moreover,
$
D+I_{-2}>0.
$
Therefore, $(D+I_{-2})(D-I_2)+1>1>0$.

Consequently, $x-w=0$.
Substituting this into (\ref{eqxw}) yields
$
y-z=0.
$

Hence
\[
f(t)=x\,p(t)+y\,q(t)=
g(t)=z\,q(t)+w\,p(t)
=y\,q(t)+x\,p(t).
\]
We conclude that every translation-invariant solution is necessarily symmetric.
Hence there are no genuinely asymmetric translation-invariant solutions.

\subsubsection{Case 3. Quadratic loss with a linear model}\label{sec:quadratic-nonseparable}

Consider the linear predictor $\mathcal F(x;t,u)=tx+u$,
and the squared-error loss
$\ell(y,z)=(y-z)^2.$
Then the empirical loss takes the form
\[
\mathcal L_N(t,u)
=
\frac1N\sum_{i=1}^{N}
(y_i-tx_i-u)^2.
\]

Expanding the square, we obtain
\[
(y_i-tx_i-u)^2
=
y_i^2+t^2x_i^2+u^2
-2tx_i y_i
-2u y_i
+2tux_i.
\]
Therefore
\[
\mathcal L_N(t,u)
=
At^2+Bu^2+Ctu+Dt+Eu+F,
\]
where
\[
A=\frac1N\sum_{i=1}^{N}x_i^2,
\qquad
B=1,
\qquad
C=\frac2N\sum_{i=1}^{N}x_i,
\]
\[
D=-\frac2N\sum_{i=1}^{N}x_i y_i,
\qquad
E=-\frac2N\sum_{i=1}^{N}y_i,
\qquad
F=\frac1N\sum_{i=1}^{N}y_i^2.
\]

Notice that the mixed coefficient $C$
is proportional to the empirical mean of the input data.

\noindent
{\bf Case \(C=0\).} 
In this case the mixed term \(tu\) disappears and the empirical loss
separates into a sum of two one-variable functions:
\[
\mathcal L_N(t,u)
=
p(t)+q(u),
\]
where $p(t)=At^2+Dt$, $q(u)=Bu^2+Eu+F$.

This loss is additive in the two interaction parameters, although
the two one-variable summands need not coincide as in Case~2.
The uniqueness conclusion of Theorem~\ref{thm:chain} for the joint
annealed boundary-law problem applies whenever its hypotheses hold.

{\bf Case \(C\neq0\).} In this case the loss contains the interaction term \(Ctu\), and therefore does not
separate into a sum of a function of \(t\) and a function of \(u\).

For comparison, consider the following parameter-only edge
Boltzmann kernel. It is the reciprocal factor $\eta_{tu}^{-1}$,
not the kernel of the joint spin-symmetric boundary-law operator:
\[
K(t,u)
=
\exp\!\Bigl(
-\beta
\bigl(
At^2+Bu^2+Ctu+Dt+Eu+F
\bigr)
\Bigr).
\]
Since the exponent is continuous on the compact square
\([0,1]^2\), the kernel is strictly positive and continuous:
$
K(t,u)>0,
\qquad
(t,u)\in[0,1]^2.
$

Hence,
\[
(T_Nf)(t)
=
\int_0^1 K(t,u)f(u)\,du
\]
defines a compact positive integral operator on
\(C([0,1])\) and on \(L^2([0,1])\).

By the Krein--Rutman theorem, the spectral radius
$r(T_N)=\lambda_N$ is a simple positive eigenvalue, and there exists a unique
(up to multiplication by a positive constant)
strictly positive eigenfunction $\psi_N(t)>0$.


Unlike the case \(C=0\), the kernel is not separable. Indeed,
\[
K(t,u)
=
e^{-\beta(At^2+Dt+F)}
e^{-\beta(Bu^2+Eu)}
e^{-\beta Ctu},
\]
and the coupling factor \(e^{-\beta Ctu}\) prevents a factorization
into a single product of a function of \(t\) and a function of \(u\). See Section \ref{sec:exp} for complementary numerical experiments for the case $C \ne 0$.





This operator belongs to a parameter-only comparison model. Its
positive eigenfunction is not, in general, the parameter marginal of the joint
annealed Gibbs measure. For the joint model, the spin-symmetric
$k=1$ boundary law is instead governed by $K_c$, whose kernel is
$\cosh(\beta\mathcal L_N(t,u))$. The numerical experiment in
Section~\ref{sec:exp} uses the joint compatibility equations,
not this parameter-only transfer operator.
     
 \subsubsection{Case $k=2$: An explicit example \(\eta_{tu}=e^{\beta(t+u)}\).}\label{k2a} 
 
  In this case the translation-invariant equations for the Cayley tree of order \(2\) are
	\begin{equation}\label{fexp}
		f(t)=
		\left(
		\frac{
			\displaystyle
			\int_0^1 e^{\beta(t+u)}f(u)\,du
			+
			\int_0^1 e^{-\beta(t+u)}g(u)\,du
		}
		{D}
		\right)^2,
	\end{equation}
	\begin{equation}\label{gexp}
		g(t)=
		\left(
		\frac{
			\displaystyle
			\int_0^1 e^{-\beta(t+u)}f(u)\,du
			+
			\int_0^1 e^{\beta(t+u)}g(u)\,du
		}
		{D}
		\right)^2,
	\end{equation}
	where
	\[
	D=
	\int_0^1 e^{-\beta u}f(u)\,du
	+
	\int_0^1 e^{\beta u}g(u)\,du .
	\]
	
	Introduce
	\begin{equation}\label{hk}
	h(t)=\sqrt{f(t)},
	\qquad
	k(t)=\sqrt{g(t)}.
	\end{equation}
	Then (\ref{fexp})--(\ref{gexp}) are equivalent to
	\begin{equation}\label{hexp}
		h(t)=
		\frac{
			e^{\beta t}\displaystyle\int_0^1 e^{\beta u}h^2(u)\,du
			+
			e^{-\beta t}\displaystyle\int_0^1 e^{-\beta u}k^2(u)\,du
		}
		{D},
	\end{equation}
	\begin{equation}\label{kexp}
		k(t)=
		\frac{
			e^{-\beta t}\displaystyle\int_0^1 e^{-\beta u}h^2(u)\,du
			+
			e^{\beta t}\displaystyle\int_0^1 e^{\beta u}k^2(u)\,du
		}
		{D}.
	\end{equation}
	
	Define
	\[
	A=\int_0^1 e^{\beta u}h^2(u)\,du,
	\qquad
	B=\int_0^1 e^{-\beta u}k^2(u)\,du,
	\]
	\[
	C=\int_0^1 e^{-\beta u}h^2(u)\,du,
	\qquad
	E=\int_0^1 e^{\beta u}k^2(u)\,du.
	\]
	
	Then (\ref{hexp})--(\ref{kexp}) become
	\begin{equation}\label{hks}
	h(t)=\frac{A e^{\beta t}+B e^{-\beta t}}{D},
	\qquad
	k(t)=\frac{C e^{-\beta t}+E e^{\beta t}}{D}.
	\end{equation}
	
	Thus every solution belongs to the finite-dimensional space
	$
	\operatorname{span}\{e^{\beta t},e^{-\beta t}\}.
	$
	
	Substituting these expressions into the definitions of
	\(A,B,C,E\) yields a closed algebraic system.
	
	For convenience, introduce
	\begin{equation}\label{Im}
	I_m=\int_0^1 e^{m\beta u}\,du
	=
	\frac{e^{m\beta}-1}{m\beta},
	\qquad m\neq0, \beta>0.
	\end{equation}
	
	Since
	\[
	h^2(u)
	=
	\frac{
		A^2 e^{2\beta u}
		+
		2AB
		+
		B^2 e^{-2\beta u}
	}{D^2}, \ \ 
	k^2(u)
	=
	\frac{
		C^2 e^{-2\beta u}
		+
		2CE
		+
		E^2 e^{2\beta u}
	}{D^2},
	\]
the original nonlinear integral system reduces to the finite-dimensional algebraic system

	\begin{equation}\label{ABCE}\begin{array}{llll}
	A=
	\frac{
		A^2 I_3
		+
		2AB I_1
		+
		B^2 I_{-1}
	}{(C+E)^2},
	\\[2mm]
	
	B=
	\frac{
		C^2 I_{-3}
		+
		2CE I_{-1}
		+
		E^2 I_1
	}{(C+E)^2},
	\\[2mm]

	C=
	\frac{
		A^2 I_1
		+
		2AB I_{-1}
		+
		B^2 I_{-3}
	}{(C+E)^2},
	\\[2mm]
	
	E=
	\frac{
		C^2 I_{-1}
		+
		2CE I_1
		+
		E^2 I_3
	}{(C+E)^2}.
	\end{array}
	\end{equation}
	Therefore the problem of finding translation-invariant joint annealed Gibbs measures for the kernel given by 
	\(\eta_{tu}=e^{\beta(t+u)}\) is reduced to solving a system of four algebraic equations in the four unknowns \(A,B,C,E\).

{\bf Case: 	$A=E, \,  B=C $.}	Assume that
	\[
	A=E, \,  B=C \ \ {\rm then} \ \ 
		D=C+E=A+B.
	\]

	Therefore the system (\ref{ABCE}) reduces to
		\[
	A=
	\frac{
		A^2 I_3
		+2AB I_1
		+B^2 I_{-1}
	}{(A+B)^2},
	\qquad
	B=
	\frac{
		A^2 I_1
		+2AB I_{-1}
		+B^2 I_{-3}
	}{(A+B)^2}.
	\]
	
Let $A=uB,$ for some $u>0$. Substituting into the first equation gives
\[
B=
\frac{
u^2I_3+2uI_1+I_{-1}
}
{u(u+1)^2}.
\]

The second equation gives
\[
B=
\frac{
u^2I_1+2uI_{-1}+I_{-3}
}
{(u+1)^2}.
\]

Equating these two values of $B$ we get 
\[
I_1u^3
+
(2I_{-1}-I_3)u^2
+
(I_{-3}-2I_1)u
-
I_{-1}
=0 .
\]

Now use $
I_m=\frac{t^m-1}{m\ln t}$, where $t=e^\beta 
$ and \(t-1\ne 0\) to obtain the cubic equation
\begin{equation}\label{cub}
P_t(u):=3t^3u^3
-
t^2(t^3+t^2+t-6)u^2
-
(6t^3-t^2-t-1)u
-
3t^2=0
\end{equation}
\begin{lemma}
For every $t>1$, the equation (\ref{cub})  has exactly one positive real solution $u>0$. 
\end{lemma}
\begin{proof}
We study the sign structure and apply Descartes' rule\footnote{Descartes rule states (see \cite{Pra}) that if the nonzero terms of a single-variable polynomial with real coefficients are ordered by descending variable exponent, then the number of positive roots of the polynomial is either equal to the number of sign changes between consecutive (nonzero) coefficients, or is less than it by an even number. A root of multiplicity $n$ is counted as $n$ roots. 	In particular, if the number of sign changes is zero or one, the number of positive roots equals the number of sign changes.} of signs together with basic continuity arguments. 
We evaluate
\[
P_t(0)=-3t^2<0,
\qquad
\lim_{u\to+\infty}P_t(u)=+\infty
\quad (\text{since } 3t^3>0).
\]
By continuity, there exists at least one $u>0$ such that $P_t(u)=0$. 
Hence the number of positive roots is at least one in for every $t>1$.

Now study sign structure of coefficients. Write
\begin{equation}\label{Pt}
P_t(u)=au^3+bu^2+cu+d,
\end{equation}
where
$$a=3t^3>0, \ \ b=-t^2(t^3+t^2+t-6), \ \ c=-(6t^3-t^2-t-1), \ \ d=-3t^2<0.$$
For $t>1$, we have
$
6t^3-t^2-t-1>0,
$
hence $c<0$ for all $t>1$.

Define
$
\phi(t)=t^3+t^2+t-6.
$

Then $\phi(t)$ is strictly increasing for $t>0$ and has a unique root
$
t_0\approx 1.34872.
$

Thus:
\[
b>0 \text{ for } 1<t<t_0,\qquad b<0 \text{ for } t>t_0.
\]

 We examine sign variations of $(a,b,c,d)$.

Case 1: $1<t<t_0$. Then $(a,b,c,d)=(+,+,-,-)$,
which has exactly one sign change. 

Case 2: $t>t_0$. Then $(a,b,c,d)=(+,-,-,-)$,
again giving exactly one sign change.  Hence the number of positive roots is at most one. Since at least one exists, there is exactly one positive root for each $t>1$.
\end{proof}
Denote by $u_*$ the unique solution of (\ref{cub}). Then, the  corresponding solution
of (\ref{ABCE}) is given by $S_*:=(u_*, 1, 1, u_*)B_*$ with 
\[
B_*=
\frac{
u_*^2I_1+2u_*I_{-1}+I_{-3}
}
{(u_*+1)^2}.
\]
The corresponding symmetric solution of integral equation is 
\begin{equation}\label{fb}	f(t)=g(t)=h^2(t)=\left(\frac{u_* e^{\beta t}+ e^{-\beta t}}{1+u_*}\right)^2,
		\end{equation}
We summarize the result here

\begin{rt}\label{res1} For every $\beta>0$ (equivalently, $t>1$), there exists a spin-flip-symmetric translation-invariant joint annealed Gibbs measure $\mu_*$ generated by the boundary-law solution $S_*$.
\end{rt}
	
{\bf Case:
	$A\ne E$ and $B\ne C$}. Above we only solved the system  (\ref{ABCE}), for \(A=E\) and \(B=C\). To see the case when these conditions are not satisfied we set
\[
D=C+E>0,
\qquad
X=A-E,
\qquad
Y=B-C.
\]

Subtracting the fourth equation of (\ref{ABCE}) from the first yields
\[
\begin{aligned}
	XD^2
	&=
	A^2I_3-E^2I_3
	+2I_1(AB-CE)
	+I_{-1}(B^2-C^2)
	\\
	&=
	I_3(A-E)(A+E)
	+2I_1\bigl[B(A-E)+E(B-C)\bigr]
	\\
	&\qquad
	+I_{-1}(B-C)(B+C).
\end{aligned}
\]
Equivalently,
\begin{equation}\label{eq1}
	\Bigl(D^2-\alpha\Bigr)X-\beta Y=0,
\end{equation}
where
\[
\alpha=I_3(A+E)+2I_1B,
\qquad
\beta=2I_1E+I_{-1}(B+C).
\]

Similarly, subtracting the third equation from the second gives

\begin{equation}\label{eq2}
	\gamma X+\Bigl(D^2+\delta\Bigr)Y=0,
\end{equation}
where
\[
\gamma=I_1(A+E)+2I_{-1}B,
\qquad
\delta=I_{-3}(B+C)+2I_{-1}E.
\]

Equations \eqref{eq1}--\eqref{eq2} form a homogeneous linear system
\begin{equation}\label{Ds}
	\begin{pmatrix}
		D^2-\alpha & -\beta\\
		\gamma & D^2+\delta
	\end{pmatrix}
	\binom{X}{Y}
	=
	\binom00 .
\end{equation}

\begin{rk} By definition, we have \(A,B,C,D,E>0\) and \(I_m>0\) for \(m=\pm1,\pm3\). Therefore, \(\gamma>0\) and \(D^2+\delta>0\). Consequently,  it follows from from \eqref{eq2} that \(X=0\) if and only if \(Y=0\).
	
	It remains to consider the case \(X\neq 0\) and \(Y\neq 0\), which is possible only if the determinant of the system \eqref{Ds},
	\[
	\Delta = (D^2-\alpha)(D^2+\delta)+\beta\gamma,
	\]
	vanishes. Below, we show that, under suitable conditions on $\beta$, the case \(\Delta=0\) is possible and solutions with $A\ne E$ and $B\ne C$ exists.

 \end{rk}

 Denote ${A\over B}=u>0$ and ${C\over E}=v>0$.
We assume that $uv\ne 1$, because, the  case $uv=1$ is equivalent to $A=E$, $B=C$. 
	
Substituting 
$
A=uB, \, C=vE
$ into the original system yields equations for \(B\) and \(E\):

\begin{equation}\label{BE}
E=\frac{v^2I_{-1}+2vI_1+I_3}{(v+1)^2}, \ \ 
B=\frac{v^2I_{-3}+2vI_{-1}+I_1}{(v+1)^2}.
\end{equation}

In the first equation of (\ref{ABCE}) 
substituting \(A=uB\) and \(C=vE\), we obtain
\[
uB=\frac{B^2\bigl(u^2I_3+2uI_1+I_{-1}\bigr)}{E^2(v+1)^2}.
\]

Cancelling \(B>0\), we get
\begin{equation}\label{BB}
B=
u\,E^2(v+1)^2\Big/\bigl(u^2I_3+2uI_1+I_{-1}\bigr).
\end{equation}

From the third equation of (\ref{ABCE}), 
substituting \(A=uB\) and \(C=vE\), we obtain
\begin{equation}\label{EE}
C=vE=\frac{B^2\bigl(u^2I_1+2uI_{-1}+I_{-3}\bigr)}{E^2(v+1)^2}.
\end{equation}
Now substituting the expression for \(E\), and $B$ given in (\ref{BE}) to (\ref{BB}) and (\ref{EE}), 
we obtain system of equations for $u$ and $v$:
\begin{equation}\label{uuvv}	\begin{aligned}
	u &= \frac{(v^2 I_{-3} + 2v I_{-1} + I_1)(u^2 I_3 + 2u I_1 + I_{-1})}{(v^2 I_{-1} + 2v I_1 + I_3)^2}, \\
	v &= \frac{(v^2 I_{-3} + 2v I_{-1} + I_1)^2 (u^2 I_1 + 2u I_{-1} + I_{-3})}{ (v^2 I_{-1} + 2v I_1 + I_3)^3}.
\end{aligned}
\end{equation}
where $I_m$, is given in (\ref{Im}) and $t>1$. Setting  $S=t^{2}+t+1$, we obtain
\begin{equation*}
I_{1}=\frac{t-1}{\beta },\quad I_{-1}=\frac{t-1}{\beta t},\quad I_{3}=\frac{%
(t-1)S}{3\beta },\quad I_{-3}=\frac{(t-1)S}{3\beta t^{3}}.
\end{equation*}%
Then, by replacing $I_{\pm 1}$ and \  $I_{\pm 3}$ in (\ref{uuvv}) we get

\begin{align}
u& =\frac{A(v,t)C(u,t)}{B(v,t)^{2}},  \label{eq:reducida1} \\
v& =\frac{A(v,t)^{2}D(u,t)}{B(v,t)^{3}},  \label{eq:reducida2}
\end{align}%
with  
\begin{align*}
A& :=A(v,t)=1+\frac{2v}{t}+\frac{S}{3t^{3}}v^{2}, \\
B& :=B(v,t)=\frac{S}{3}+2v+\frac{v^{2}}{t}, \\
C& :=C(u,t)=\frac{1}{t}+2u+\frac{S}{3}u^{2}, \\
D& :=D(u,t)=\frac{S}{3t^{3}}+\frac{2u}{t}+u^{2}.
\end{align*}

Therefore,$\allowbreak $  
\begin{equation*}
uB^{2}-AC=0,\qquad vB^{3}-A^{2}D=0.
\end{equation*}%
Since $C=\frac{1}{t}+2u+\frac{Su^{2}}{3}$ and $uB^{2}-AC=0,$ we have $%
uB^{2}-A\left( \frac{1}{t}+2u+\frac{Su^{2}}{3}\right) =0,$ so that 
\begin{equation*}
\frac{SA}{3}u^{2}+\left( 2A-B^{2}\right) u+\frac{A}{t}=0.
\end{equation*}%
Let $a_{1}=\frac{SA}{3};$ $b_{1}=2A-B^{2}$ and $c_{1}=\frac{A}{t}.$ Then,
$
a_{1}u^{2}+b_{1}u+c_{1}=0.$
Similarly, since $vB^{3}-A^{2}D=0$ and $D=\frac{S}{3t^{3}}+\frac{2u}{t}+u^{2}
$ we have $vB^{3}-A^{2}\left( \frac{S}{3t^{3}}+\frac{2u}{t}+u^{2}\right) =0,$
so that 
\begin{equation*}
vB^{3}-\frac{A^{2}S}{3t^{3}}-\frac{2A^{2}}{t}u-A^{2}u^{2}=0.
\end{equation*}%
Thus, $A^{2}u^{2}+\frac{2A^{2}}{t}u+\left( \frac{A^{2}S}{3t^{3}}%
-vB^{3}\right) =0$ and setting%
\begin{equation*}
a_{2}=A^{2},\qquad b_{2}=\frac{2A^{2}}{t},\qquad c_{2}=\frac{A^{2}S}{3t^{3}}%
-vB^{3},
\end{equation*}%
we have $a_{2}u^{2}+b_{2}u+c_{2}=0.$

We conclude that the solutions are derived from the equations
\begin{equation}
a_{1}u^{2}+b_{1}u+c_{1}=0\text{ and \ }a_{2}u^{2}+b_{2}u+c_{2}=0\text{ }
\label{sistemas}
\end{equation}%
This system can be written as 
\begin{equation*}
\left( 
\begin{array}{cc}
a_{1} & a_{2} \\ 
b_{1} & b_{2}%
\end{array}%
\right) \left( 
\begin{array}{c}
u^{2} \\ 
u%
\end{array}%
\right) =\left( 
\begin{array}{c}
-c_{1} \\ 
-c_{2}%
\end{array}%
\right) ,
\end{equation*}

with 
\begin{align*}
a_{1}& =\frac{SA}{3}, & b_{1}& =2A-B^{2}, & c_{1}& =\frac{A}{t}, \\
a_{2}& =A^{2}, & b_{2}& =\frac{2A^{2}}{t}, & c_{2}& =\frac{A^{2}S}{3t^{3}}%
-vB^{3}.
\end{align*}

We claim that $a_{1}b_{2}-b_{1}a_{2}>0.$ Indeed,%
\begin{equation*}
a_{1}b_{2}-a_{2}b_{1}=\frac{SA}{3}\frac{2A^{2}}{t}-A^{2}(2A-B^{2})=A^{2}%
\left[ B^{2}+2A\left( \frac{S}{3t}-1\right) \right] .
\end{equation*}

Since $S=t^{2}+t+1,$ we have that $\frac{S}{3t}=\frac{t^{2}+t+1}{3t}=\frac{%
t+1+\frac{1}{t}}{3}.$ Moreover $t>1$, and hence $t+1+\frac{1}{t}>3,$ so that 
$\frac{S}{3t}-1>0,$ and the claim follows. 

Consequently (\ref{sistemas}) has a unique solution that can be written as  
has a unique solution  
\begin{equation*}
\left( 
\begin{array}{c}
u^{2} \\ 
u%
\end{array}%
\right) =\left( 
\begin{array}{cc}
\frac{b_{2}}{a_{1}b_{2}-a_{2}b_{1}} & -\frac{a_{2}}{a_{1}b_{2}-a_{2}b_{1}}
\\ 
-\frac{b_{1}}{a_{1}b_{2}-a_{2}b_{1}} & \frac{a_{1}}{a_{1}b_{2}-a_{2}b_{1}}%
\end{array}%
\right) \left( 
\begin{array}{c}
-c_{1} \\ 
-c_{2}%
\end{array}%
\right) =\allowbreak \left( 
\begin{array}{c}
\frac{a_{2}c_{2}-b_{2}c_{1}}{a_{1}b_{2}-a_{2}b_{1}} \\ 
\frac{b_{1}c_{1}-a_{1}c_{2}}{a_{1}b_{2}-a_{2}b_{1}}%
\end{array}%
\right) .
\end{equation*}%
if and only if the following consistence condition is fulfilled%
\begin{equation*}
\left( \frac{b_{1}c_{1}-a_{1}c_{2}}{a_{1}b_{2}-a_{2}b_{1}}\right) ^{2}=\frac{%
a_{2}c_{2}-b_{2}c_{1}}{a_{1}b_{2}-a_{2}b_{1}},
\end{equation*}%
or equivalently $(b_{1}c_{1}-a_{1}c_{2})^{2}=(a_{1}b_{2}-a_{2}b_{1})(a_{2}c_{2}-b_{2}c_{1}).
$

$\allowbreak $Now define 
\begin{equation}
F(v,t)=(a_{2}c_{1}-a_{1}c_{2})^{2}-(a_{1}b_{2}-a_{2}b_{1})(b_{1}c_{2}-b_{2}c_{1}).
\end{equation}%
Then the compatibility condition is 
\begin{equation}
F(v,t)=0.
\end{equation}
By a computer one gets the factorization of $F(v,t)$ 
\begin{equation}\label{factor}
F(v,t)={1\over 81 t^8}Y(v,t)Q_3(v,t)Q_8(v,t),
\end{equation}
where
\begin{eqnarray}
Y(v,t)&=&(3v^2+6tv+t^3+t^2+t)^2>0.\\
Q_3(v,t)&=&-v^3P_t(1/v)=3t^{2}v^{3}+(6t^{3}-t^{2}-t-1)v^{2}+(t^{5}+t^{4}+t^{3}-6t^{2})v-3t^{3}
\end{eqnarray}
\[
\begin{aligned}
Q_8(v,t)=\;&-243t^{5}v^{8}
+\left(27t^{10}+54t^{9}+81t^{8}+54t^{7}-1917t^{6}\right)v^{7}\\
&+\left(162t^{11}+333t^{10}+513t^{9}+54t^{8}
-5985t^{7}-54t^{6}+189t^{5}+225t^{4}-54t^{3}\right)v^{6}\\
&+\left(27t^{13}+405t^{12}+846t^{11}+1323t^{10}
-810t^{9}-9117t^{8}-1053t^{7}
+2430t^{6}+576t^{5}-216t^{4}\right)v^{5}\\
&+\left(108t^{14}+546t^{13}+1140t^{12}+1653t^{11}
-2372t^{10}-6745t^{9}-3660t^{8}
+8577t^{7}\right.\\
&\qquad\left.
-300t^{6}-150t^{5}-18t^{4}-6t^{3}
+24t^{2}-10t-2\right)v^{4}\\
&+\left(9t^{16}+144t^{15}+426t^{14}+876t^{13}
+831t^{12}-2112t^{11}-2694t^{10}\right.\\
&\qquad\left.
-3432t^{9}+11481t^{8}-948t^{7}
+204t^{6}-456t^{5}+408t^{4}
-96t^{3}-24t^{2}\right)v^{3}\\
&+\left(18t^{17}+73t^{16}+185t^{15}+357t^{14}
-12t^{13}-411t^{12}-1515t^{11}\right.\\
&\qquad\left.
+873t^{10}+4656t^{9}+3021t^{8}
-2839t^{7}+1333t^{6}
-18t^{5}-120t^{4}-12t^{3}\right)v^{2}\\
&+\left(t^{19}+5t^{18}+15t^{17}+36t^{16}
+33t^{15}+3t^{14}-111t^{13}\right.\\
&\qquad\left.
-216t^{12}+903t^{11}+785t^{10}
+2581t^{9}-2550t^{8}
+1476t^{7}-216t^{6}-72t^{5}\right)v\\
&-18t^{14}+54t^{13}+90t^{12}+144t^{11}
+486t^{10}-558t^{9}+360t^{8}
-54t^{7}-18t^{6}.
\end{aligned}
\]

Since $Y(v,t)>0$, we should have $Q_3=0$ or $Q_8=0$. The case $Q_3=0$ gives symmetric solutions which were discussed by roots of $P_t(1/v)$, $u=1/v$ defined by (\ref{Pt}). So it remains to solve equation $Q_8(v,t)=0$. 

\begin{lemma} For \(t>1\), let  \(N_+(t)\) denotes the number of positive roots of the 
equation \(Q_8(v,t)=0\). Then,  $N_+(t)\in\{1, 2\}.$
\end{lemma}
\begin{proof} Write $Q_8(v,t)$ as $Q_8(v,t)=\sum_{i=0}^8k_i(t)v^i$. 
We study the signs of the coefficients \(k_i(t)\) and then apply Descartes' rule of signs. Since each \(k_i(t)\), \(i=0,\ldots,8\), is an explicit polynomial in \(t\), its zeros can be computed. In particular, the relevant critical values are $t_{i,c}$, $i=0,4,5,6,7$ where 
$$t_{4,c}\approx 1.351, \ \ t_{5,c}\approx 1.613, \ \ t_{6,c}\approx 1.857, \ \ t_{7,c}\approx 2.281, \ \ t_{0,c}\approx 4.6545$$ and the  following table is obtained.

\begin{table}[h!]
	\centering
	\begin{tabular}{|c|c|c|}
		\hline
		Interval in $t$ & Sign of $(k_8(t), \dots, k_0(t))$ & \(N_+(t)\) \\ 
		\hline
		$(1, t_{4,c})$ & $(-,-,-,-,-,+,+,+,+)$ & 1\\ 
		\hline
		$(t_{4,c}, t_{5,c})$ & $(-,-,-,-,+,+,+,+,+)$&  1 \\ 
		\hline
		$(t_{5,c}, t_{6,c})$ & $(-,-,-,+,+,+,+,+,+)$&  1 \\ 
		\hline
		$(t_{6,c}, t_{7,c})$ & $(-,-,+,+,+,+,+,+,+)$ &  1\\ 
        \hline
		$(t_{7,c}, t_{0,c})$ & $(-,+,+,+,+,+,+,+,+)$& 1 \\ 
		\hline
		$(t_{0,c}, +\infty)$ & $(-,+,+,+,+,+,+,+,-)$&  2 \\ 
		\hline
	\end{tabular}
\end{table} 

For the case $t\in(1,t_{0,c})$, we have $
k_8(t)<0,\, k_0(t)>0.$
Therefore, $ Q_8(0,t)=:\lim_{v\to 0} Q_8(v,t)>0$ and $\lim_{v\to +\infty} Q_8(v,t)<0$. 
Hence, by the intermediate value theorem, there exists at least one
positive solution $v>0$ of the equation $Q_8(v,t)=0$. Moreover, applying Descartes' rule of signs, we conclude that
$Q_8(v,t)$ has exactly one positive root (in case $t\in(1,t_{0,c})$).

For the case $t>t_{0,c}$, we have $Q_8(0,t)<0$. Moreover, $Q_8(1,t)>0$. Indeed, 
the factorization of \(Q_8(1,t)\) is $Q_8(1,t)=(t-1)R_{18}(t)$,
where 
\[
\begin{aligned}
R_{18}(t)=\;&t^{18}+6t^{17}+39t^{16}+157t^{15}
+519t^{14}+1395t^{13}\\
&+2775t^{12}+4614t^{11}+4695t^{10}
+3456t^9\\
&-285t^8-615t^7-1387t^6
+75t^5\\
&-99t^4+180t^3+12t^2+12t+2 .
\end{aligned}
\]

It is easy to see that $R_{18}(t)>0$, for $t>t_{0,c}$ because the
high-degree positive terms dominate the few negative coefficients.

Therefore, by the intermediate value theorem, there exists at least one positive root of
$Q_8(v,t)=0$ in the interval $(0,1)$. Moreover, by Descartes' rule of signs, the
polynomial has at most two positive roots (counting multiplicities). Since we have
already established the existence of one positive root, the possibility of having
zero positive roots is excluded. Hence, the only possible number of positive roots is
two. See Figure \ref{fig:toy-octic} for the visualization of bifurcation.
\end{proof}

Here some numerical analysis for positive roots of $Q_8(v,t)$:
\begin{itemize}
\item for $t=1.63$ one solution, namely  $v\approx 2.623$, 
\item for $t=2$ one solution, namely  $v\approx 7.315$, 
\item for critical $t=t_{0.c}$ we have {\bf two} solutions $v_1\approx 0.00000025$ and $v_2\approx 376.98$.
\item for $t=5$  there are two solutions $v\approx0.0002, v\approx524.21$
\end{itemize}

\begin{figure}[t]\centering
\includegraphics[width=.6\linewidth]{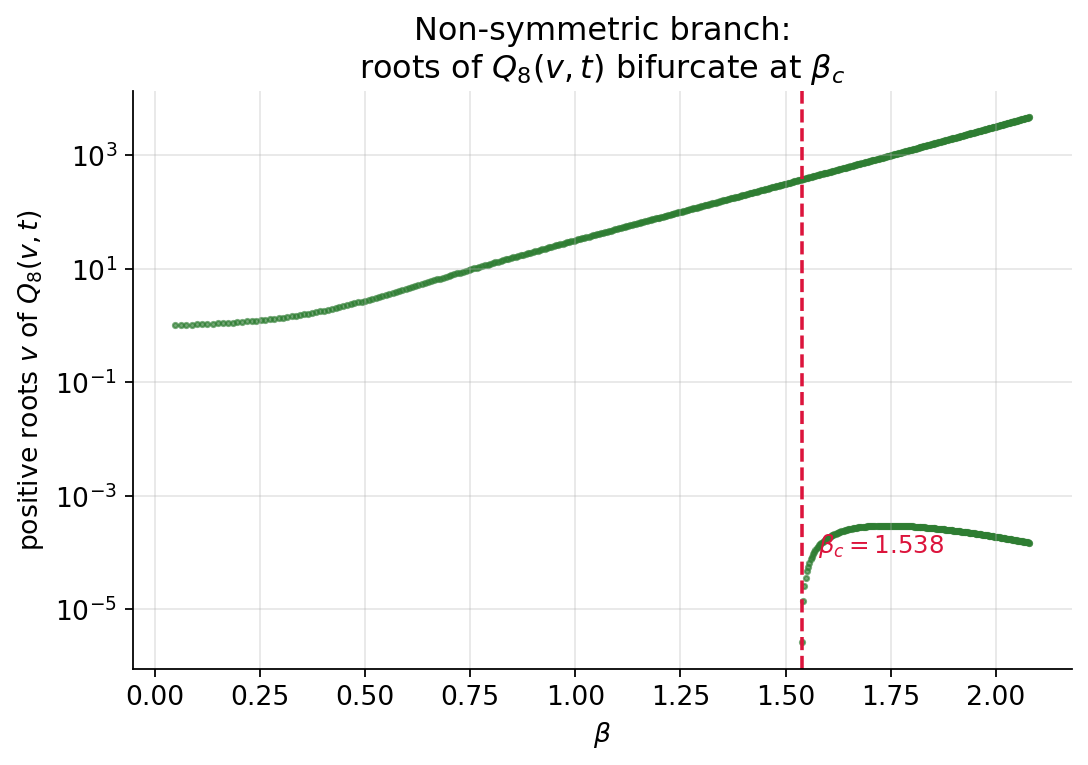}
\caption{Positive roots $v$ of the octic $Q_8(v,t)$ versus $\beta$. A second
branch appears at $\beta_c$, the mechanism behind the $2\to3$ jump.}
\label{fig:toy-octic}
\end{figure}

Consequently, we obtain the following result.

\begin{rt}\label{res2}
Let \(\beta=\ln t>0\) and $\beta_c=\ln(t_{0,c})$. The algebraic
analysis above reports the following number $n_{\mathrm{alg}}(\beta)$
of positive candidate branches of the reduced octic $Q_8(v,t)=0$:

\[
n_{\mathrm{alg}}(\beta)=
\begin{cases}
1, & \beta<\beta_c,\\[2mm]
2, & \beta\geq \beta_c.
\end{cases}
\]
\end{rt}

\begin{rk}[Algebraic branches and individual measures]
The root count in Result~\ref{res2} refers to the reduced polynomial
equation. To obtain a boundary law, each root must yield positive
functions $f,g$ satisfying the unreduced equations, including any
conditions imposed before elimination. 
The reconstruction requires the correct solution of the two
quadratic equations~\eqref{sistemas}. With
$\Delta_u=a_1b_2-a_2b_1>0$, it is
\[
 u=\frac{a_2c_1-a_1c_2}{\Delta_u},\qquad
 u^2=\frac{b_1c_2-b_2c_1}{\Delta_u}.
\]
Both identities and $u>0$ must hold. 

For example, at $t=2$ ($\beta=\ln2$), a numerical reconstruction
 gives the nonsymmetric pair
\[
 (u_-,v_-)\approx(0.1550403075,\,0.03629706545),\qquad
 (u_+,v_+)\approx(27.55043659,\,6.449935607),
\]
with $(u_+,v_+)=(1/v_-,1/u_-)$. These positive solutions,
together with the symmetric solution, give at least three distinct
individual measures. 
\end{rk}

For given solution $(u,v)$ with $uv\ne 1$, the non-symmetric solution of the integral equation 
found by formulas (\ref{hk}) and (\ref{hks}): 
\begin{equation}\label{ft}
	f(t)=\left(
\frac{
v^2I_{-3}+2vI_{-1}+I_1
}
{
v^2I_{-1}+2vI_1+I_3
}
\right)^2\cdot
\left({ue^{\beta t}+e^{-\beta t}\over v+1}
\right)^2, \ \ 
	g(t)	=
			\left(
		{v e^{-\beta t}+e^{\beta t}\over 
		v+1}\right)^2.
\end{equation}

\section{Spin prediction under selected annealed Gibbs measures}
\label{sec:prediction}

We describe the joint finite-volume densities of the selected
translation-invariant annealed Gibbs measures on interaction parameters
and spins constructed in Subsection~\ref{k2a}. Their values are densities
relative to $\lambda_{V_n}$ and counting measure on $\Sigma_{V_n}$;
\eqref{muu} converts them into probabilities of cylinder events.
Since the balls $V_n$ exhaust the tree, these marginals determine the
joint measure.

For binary-label prediction, assign input samples to distinct tree
vertices $v_1,\ldots,v_N$, as in Remark~\ref{rk:observations}, and let
$\Lambda_N=\{v_1,\ldots,v_N\}$. Condition on
$\mathcal E=\{\varphi(v_i)=y_i:\ i=1,\ldots,N\}$, with
$y_i\in\{-1,+1\}$. For a query $x\notin\Lambda_N$, choose any
$n$ such that $\Lambda_N\cup\{x\}\subseteq V_n$. The observation
count $N$ is not the radius of this ball.

For boundary functions $h$ generating a selected joint measure $\mu$,
write

\[
 U_n^h(\sigma_n,\varphi_n)
 =\exp\left[-\beta H(\sigma_n,\varphi_n)
 +\sum_{z\in W_n}h_{\sigma_n(z),\varphi_n(z),z}\right].
\]

The constrained partition functions are

\begin{equation}\label{eq:prediction-partition}
 Z_{x,n}^{h}(s\mid\mathcal E)
 =\sum_{\substack{\varphi_n\in\Sigma_{V_n}\\
 \varphi_n(x)=s,\ \varphi_n(v_i)=y_i\ (1\leq i\leq N)}}
 \int_{\Omega_{V_n}}U_n^h(\sigma_n,\varphi_n)
 \lambda_{V_n}(d\sigma_n),\qquad s\in\{-1,+1\}.
\end{equation}

Thus all interaction parameters are integrated over and all unobserved
spins are summed over. The finite-volume normalizer cancels in

\begin{equation}\label{eq:prediction-ratio}
 \mu(\varphi(x)=s\mid\mathcal E)
 =\frac{Z_{x,n}^{h}(s\mid\mathcal E)}
 {Z_{x,n}^{h}(+1\mid\mathcal E)+Z_{x,n}^{h}(-1\mid\mathcal E)}.
\end{equation}

Compatibility makes this ratio independent of the chosen admissible
$n$. The spin prediction is
\[
 \widehat\varphi_\mu(x)
 =\mathbb E_\mu[\varphi(x)\mid\mathcal E]
 =\frac{Z_{x,n}^{h}(+1\mid\mathcal E)-Z_{x,n}^{h}(-1\mid\mathcal E)}
 {Z_{x,n}^{h}(+1\mid\mathcal E)+Z_{x,n}^{h}(-1\mid\mathcal E)}.
\]
The following formulas specify the weight $U_n^h$ for the symmetric
and non-symmetric boundary laws; the prediction is then computed by
\eqref{eq:prediction-partition}--\eqref{eq:prediction-ratio}.

{\bf Case $\mu_*$:} Consider the translation-invariant solution (\ref{fb}) and corresponding symmetric measure $\mu_*$.
By the definition of the boundary functions,
\[
f(t,x)=\exp(h_{t,+,x}-h_{0,-,x}),
\qquad
g(t,x)=\exp(h_{t,-,x}-h_{0,-,x}),
\]
and since \(f=g\), we obtain
\[
h_{t,+,x}-h_{0,-,x}
=
h_{t,-,x}-h_{0,-,x}
=
\ln f(t).
\]

Therefore,
\[
h_{t,\pm,x}-h_{0,-,x}
=
2\ln
\left(
\frac{u_*e^{\beta t}+e^{-\beta t}}
{1+u_*}
\right).
\]

The finite-volume Gibbs distribution is given (\ref{uy2}) and for the kernel
$\eta_{tu}=e^{\beta(t+u)}$, 
the interaction function is
$
\xi(t,u)=t+u,
$
and hence
 the Hamiltonian
\[
H(\sigma_n,\varphi_n)
=
-\sum_{\langle x,y\rangle\in L_n}
\big(\sigma(x)+\sigma(y)\big)\varphi(x)\varphi(y).
\]

Hence (absorbing constants into \(Z_n\)):
$$
\sum_{x\in W_n} h_{\sigma(x),\varphi(x),x}
=
2\sum_{x\in W_n}
\ln\left(
\frac{u_* e^{\beta \sigma(x)}+e^{-\beta \sigma(x)}}{1+u_*}
\right).$$

Therefore the finite-volume joint density is
$$
\mu^{(n)}_*(\sigma_n,\varphi_n) =
\frac{1}{Z_n}
\exp\Bigg[
\beta \sum_{\langle x,y\rangle\in L_n}
\big(\sigma(x)+\sigma(y)\big)\varphi(x)\varphi(y)
+
2\sum_{x\in W_n}
\ln\left(
\frac{u_* e^{\beta \sigma(x)}+e^{-\beta \sigma(x)}}{1+u_*}
\right)
\Bigg]=$$
$$
\frac{1}{Z_n}
\prod_{\langle x,y\rangle\in L_n}
\exp\Big(\beta(\sigma(x)+\sigma(y))\varphi(x)\varphi(y)\Big)
\prod_{x\in W_n}
\left(
\frac{u_* e^{\beta \sigma(x)}+e^{-\beta \sigma(x)}}{1+u_*}
\right)^2=
$$
\begin{equation}\label{m*}
\frac{1}{Z_n}
\prod_{\langle x,y\rangle\in L_n}
q^{\Big((\sigma(x)+\sigma(y))\varphi(x)\varphi(y)\Big)}
\prod_{x\in W_n}
\left(
\frac{u_* q^{\sigma(x)}+q^{-\sigma(x)}}{1+u_*}
\right)^2,
\end{equation}
where $\sigma(z)\in[0,1]$, $\varphi(z)\in\{-1,+1\}$, $q=e^\beta$.\\

{\bf Spin prediction under the symmetric joint annealed Gibbs measure.}
For $\mu_*$, the unnormalized finite-volume weight in
\eqref{eq:prediction-partition} is
\[
 U_n^*(\sigma_n,\varphi_n)
 =\prod_{\langle z,w\rangle\in L_n}
 q^{(\sigma_n(z)+\sigma_n(w))\varphi_n(z)\varphi_n(w)}
 \prod_{z\in W_n}
 \left(\frac{u_*q^{\sigma_n(z)}+q^{-\sigma_n(z)}}{1+u_*}\right)^2.
\]

Insert this weight into the constrained spin sum and parameter integral
in \eqref{eq:prediction-partition}. The difference of the resulting two conditional probabilities gives
$\widehat\varphi_{\mu_*}(x)$ for every unobserved vertex $x$.
The coefficient $u_*$ describes the symmetric boundary law; it is
neither a local interaction parameter nor a spin value. Although
$\mu_*$ is invariant under global spin reversal, conditioning on
observed labels can break that symmetry. Symmetry of the unconditioned
joint measure therefore does not imply a zero or uniform prediction
for every observation event.

{\bf Non-symmetric case.} 
Consider a joint annealed Gibbs measure constructed by a non-symmetric solution of (\ref{uuvv}),
that is \(uv\neq 1\).

Define
\[
R(v,q)=
\frac{
\frac{v^2}{3}(1-q^{-3}) + 2v(1-q^{-1}) + (q-1)
}{
v^2(1-q^{-1}) + 2v(q-1) + \frac{1}{3}(q^3-1)
}.
\]

Then the non-symmetric boundary laws (\ref{ft}) become
\begin{equation}\label{ft-full-new}
f(t)=
R(v,q)^2
\left(
\frac{u q^{t}+q^{-t}}{v+1}
\right)^2,
\end{equation}

\begin{equation}\label{gt-full-new}
g(t)=
\left(
\frac{v q^{-t}+q^{t}}{v+1}
\right)^2.
\end{equation}
 
The boundary functions satisfy

\[
h_{t,+,x}=\ln f(t)+h_{0,-,x},
\qquad
h_{t,-,x}=\ln g(t)+h_{0,-,x}.
\]

Therefore,
\begin{equation}\label{boundary-sum-new}
\sum_{x\in W_n}h_{\sigma(x),\varphi(x),x}
=
\sum_{x\in W_n}
\Big(
\mathbf{1}_{\{\varphi(x)=+\}}\ln f(\sigma(x))
+
\mathbf{1}_{\{\varphi(x)=-\}}\ln g(\sigma(x))
\Big)
+ C_n,
\end{equation}
where \(C_n\) is absorbed into \(Z_n\).

We compute
\[
\ln f(t)
=
2\ln R(v,q)
+
2\ln\left(\frac{u q^{t}+q^{-t}}{v+1}\right),
\]
\[
\ln g(t)
=
2\ln\left(\frac{v q^{-t}+q^{t}}{v+1}\right).
\]

Substituting into the finite-volume joint density gives:
$$
\mu^{(n)}(\sigma_n,\varphi_n)
=$$
$$\frac{1}{Z_n}
\prod_{\langle x,y\rangle\in L_n}
q^{(\sigma(x)+\sigma(y))\varphi(x)\varphi(y)}
\prod_{x\in W_n}\left[
\left(R(v,q)\frac{u q^{\sigma(x)}+q^{-\sigma(x)}}{v+1}\right)^{2\mathbf{1}_{\{\varphi(x)=+\}}}
\left(\frac{v q^{-\sigma(x)}+q^{\sigma(x)}}{v+1}\right)^{2\mathbf{1}_{\{\varphi(x)=-\}}}\right],
$$
where $u,v$ is a solution of (\ref{uuvv}), with $uv\ne 1$.\\ 

{\bf Spin prediction under a non-symmetric joint annealed Gibbs measure.}
Fix an admissible boundary law $(u,v)$ and the associated joint
measure $\mu_{u,v}$. Its unnormalized weight is
\[
\begin{aligned}
 U_n^{u,v}(\sigma_n,\varphi_n)
 &=\prod_{\langle z,w\rangle\in L_n}
 q^{(\sigma_n(z)+\sigma_n(w))\varphi_n(z)\varphi_n(w)}\\
 &\quad\times\prod_{z\in W_n}
 \left(R(v,q)\frac{u q^{\sigma_n(z)}+q^{-\sigma_n(z)}}{v+1}\right)
 ^{2\mathbf1_{\{\varphi_n(z)=+1\}}}
 \left(\frac{v q^{-\sigma_n(z)}+q^{\sigma_n(z)}}{v+1}\right)
 ^{2\mathbf1_{\{\varphi_n(z)=-1\}}}.
\end{aligned}
\]

Using this weight in \eqref{eq:prediction-partition} defines
$Z_{x,n}^{u,v}(s\mid\mathcal E)$. Equations
\eqref{eq:prediction-ratio} and the following spin expectation then
give the conditional prediction under $\mu_{u,v}$:
\[
 \widehat\varphi_{\mu_{u,v}}(x)
 =\frac{Z_{x,n}^{u,v}(+1\mid\mathcal E)-Z_{x,n}^{u,v}(-1\mid\mathcal E)}
 {Z_{x,n}^{u,v}(+1\mid\mathcal E)+Z_{x,n}^{u,v}(-1\mid\mathcal E)}.
\]
The boundary-law coefficients $(u,v)$ must be distinguished from the
interaction parameters $\sigma(z)$ and the spins $\varphi(z)$.

\begin{rk}[Joint multiplicity and conditional spin prediction]
The symmetric and non-symmetric measures above are annealed Gibbs
measures on interaction parameters and spins. Summing over spins gives
their parameter marginals; fixing the parameter configuration gives
conditional Gibbs measures on spins for almost every configuration
under the corresponding marginal. These operations produce different
objects from the joint measure itself.

For each selected joint measure, conditioning on $\mathcal E$ gives
one well-defined conditional expectation. Different selected measures
can give different spin predictions, but their predictions or predicted
labels need not differ at every vertex. 
\end{rk}
\section{Numerical experiments for the non-separable case: $C \ne 0$}\label{sec:exp}

We now replace the model kernel by a kernel generated, in the manner of
Section \ref{sec:quadratic-nonseparable}, from a dataset.

{\bf The dataset and its induced kernel.}
Let $\{(x_i,y_i)\}_{i=1}^{N}$ with $N=400$ be points $x_i\in\mathbb R^{5}$
and labels $y_i\in\{-1,+1\}$, the two classes being generated as samples of
$\mathcal N(\mu,I_5)$ and $\mathcal N(-\mu,I_5)$ in equal proportion, with
$\mu=(\tfrac{13}{10},0,0,0,0)$ (Figure \ref{fig:data}). To each point we associate a fixed scalar feature
$r_i\in[0,1]$, obtained by projecting $x_i$ onto the direction
$w=(\bar x_{+}-\bar x_{-})/\lVert \bar x_{+}-\bar x_{-}\rVert$, where
$\bar x_{\pm}$ are the class means, and mapping the projection affinely into
$[0,1]$. The feature $r(x)$ thus records the position of $x$ relative to
the separating hyperplane.

These observed features are covariates used to construct the coupling function; they are distinct from the random tree interaction parameters $\sigma(z)$ and the arguments $t,u$.

With the affine pair-predictor $\mathcal F(x_i;t,u)=r_i\,t+u$ and the
squared loss $\ell(y,\hat y)=(y-\hat y)^2$, the empirical interaction loss is,
as in Case~3,

\[
\mathcal L_N(t,u)
=\frac1N\sum_{i=1}^{N}\bigl(y_i-r_i\,t-u\bigr)^2
=At^2+Bu^2+Ctu+Dt+Eu+F,
\]
with coefficients
\[
A=\frac1N\sum_ir_i^2,\quad
B=1,\quad
C=\frac2N\sum_ir_i,\quad
D=-\frac2N\sum_ir_i y_i,\quad
E=-\frac2N\sum_i y_i,\quad
F=\frac1N\sum_i y_i^2 .
\]

For the present dataset these take the values
$A\approx0.31$, $B=1$, $C\approx1.01$, $D\approx-0.39$, $E\approx0$,
$F=1$. Since $C\neq0$, the kernel $\eta_{tu}=e^{\beta\mathcal L_N(t,u)}$ does
not factorize as in Case~2, and one is in the non-separable situation of
Case~3. Moreover $\mathcal L_N(t,u)\geq\min_{[0,1]^2}\mathcal L_N\approx0.88>0$,
so that the positivity hypothesis $\mathcal L_N(t,u)>0$ of Theorem~1 is
satisfied, and its uniqueness assertion applies to the normalized positive translation-invariant boundary-law solution for $k=1$. See Figures \ref{fig:data-loss} and \ref{fig:data-kernel} for the visualizations of the loss function and the corresponding data kernel.

{\bf The translation-invariant equations for $k=2$.}
For the tree of order $k=2$ we consider the translation-invariant system
\eqref{fexp}--\eqref{gexp} with the data-induced kernel
$\eta_{tu}=e^{\beta\mathcal L_N(t,u)}$, namely
\[
f(t)=\left(\frac{\int_0^1\eta_{tu}f(u)\,du+\int_0^1\eta_{tu}^{-1}g(u)\,du}{D}\right)^2,
\quad
g(t)=\left(\frac{\int_0^1\eta_{tu}^{-1}f(u)\,du+\int_0^1\eta_{tu}g(u)\,du}{D}\right)^2,
\]
with $D=\int_0^1\eta_{0u}^{-1}f(u)\,du+\int_0^1\eta_{0u}g(u)\,du$. Each of these
equations contains integrals of the unknown functions over $[0,1]$, and so
cannot be treated as a finite system as it stands. We approximate the integrals by quadrature and solve the resulting finite system numerically. Quadrature error, stopping tolerances, and the finite search over initial conditions all affect this computation.

\begin{figure}[t]\centering
\includegraphics[width=.9\linewidth]{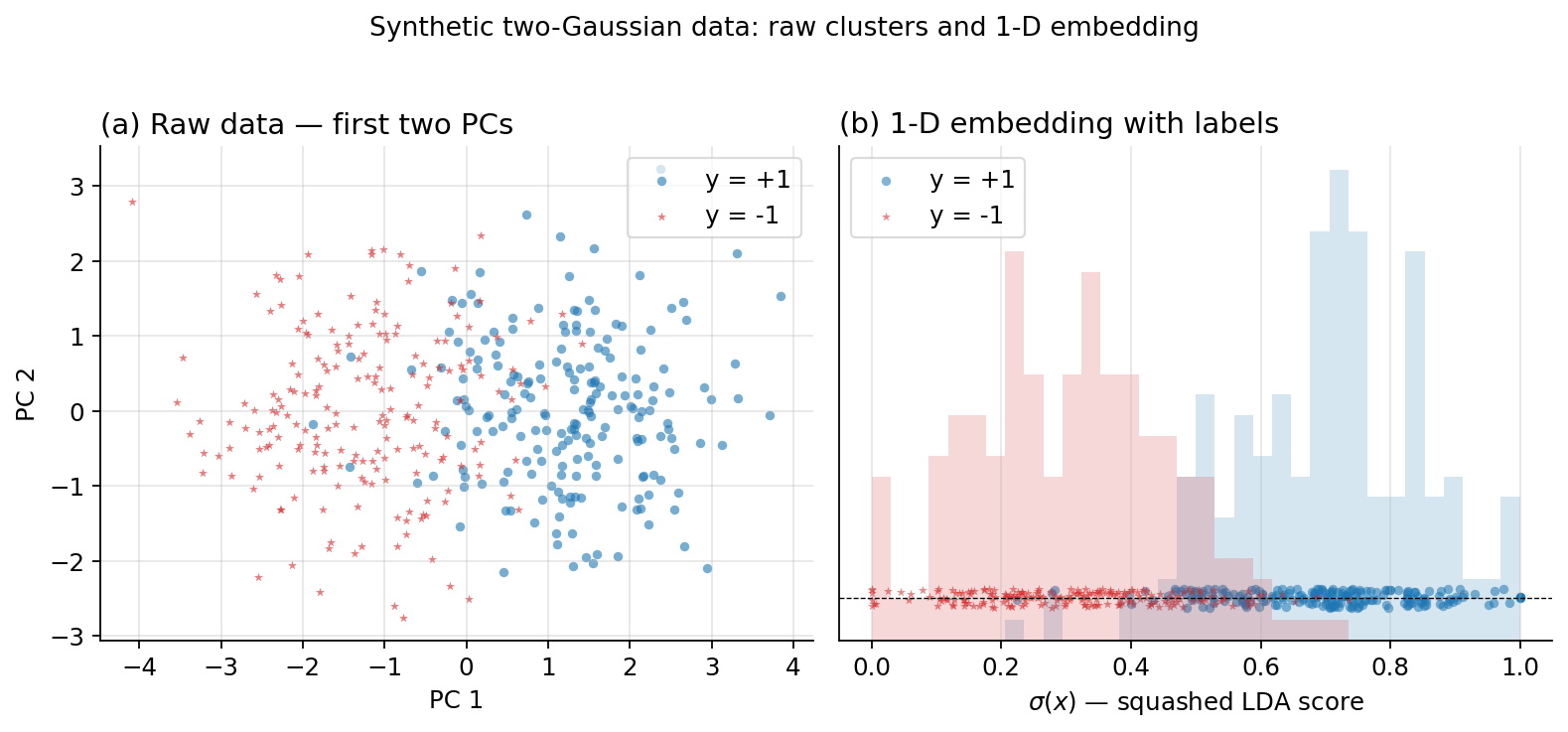}
\caption{Synthetic two-Gaussian data (left, first two principal components) and
the one-dimensional embedding $\sigma(x)$ with class-conditional histograms
(right).}
\label{fig:data}
\end{figure}

\begin{figure}[t]\centering
\includegraphics[width=.6\linewidth]{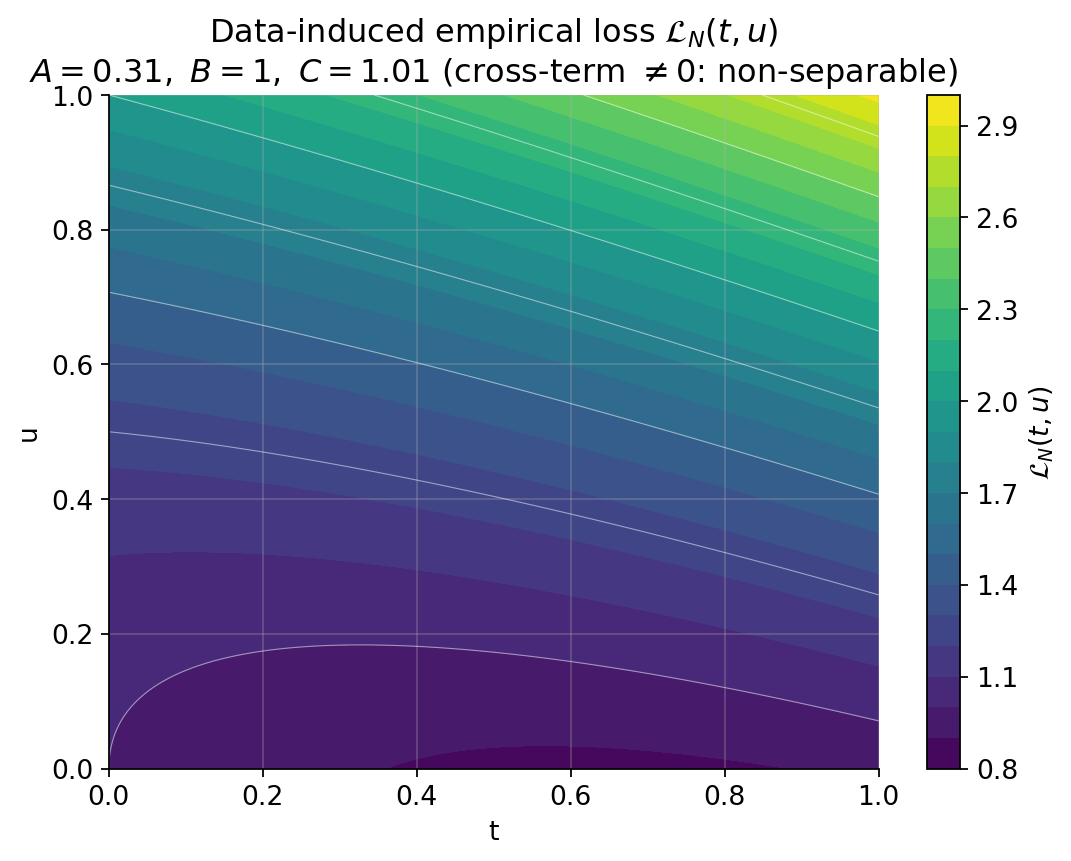}
\caption{Data-induced empirical loss $\mathcal L_N(t,u)$. The non-zero cross term
($C=1.01$) makes the kernel non-separable (Case~3); the surface is strictly
positive ($\min=0.88$).}
\label{fig:data-loss}
\end{figure}

\begin{figure}[t]\centering
\includegraphics[width=.95\linewidth]{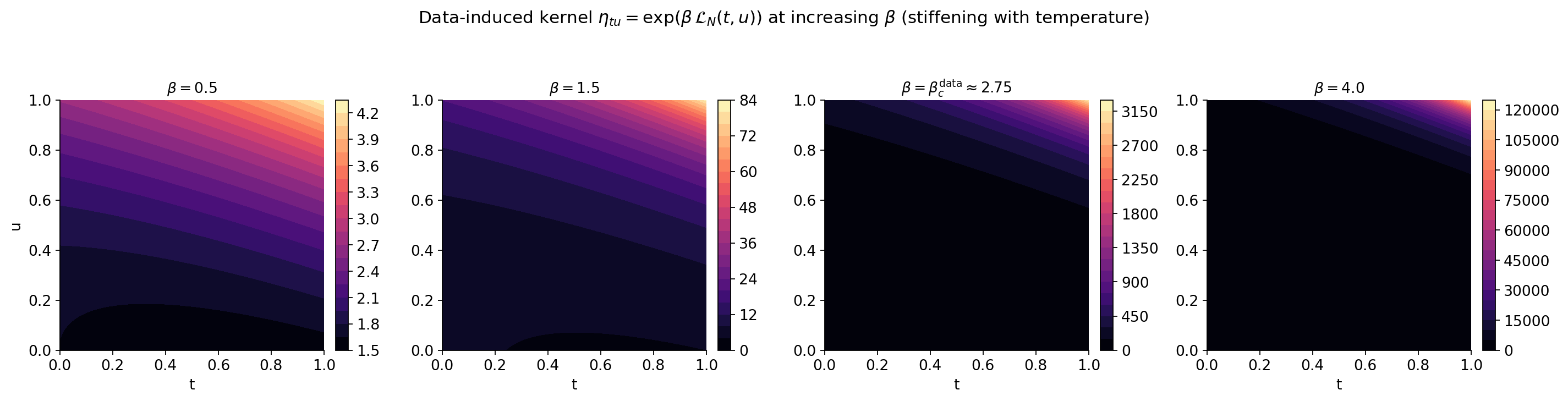}
\caption{Data-induced kernel $\eta_{tu}=e^{\beta\mathcal L_N(t,u)}$ at increasing
$\beta$. As $\beta$ grows the kernel stiffens and concentrates near a corner of
$[0,1]^2$, which is the source of the high-$\beta$ numerical difficulty.}
\label{fig:data-kernel}
\end{figure}
{\bf Discretization of the integrals.}
A quadrature rule replaces an integral by a finite weighted sum,
$\int_0^1\psi(u)\,du\approx\sum_{j=1}^{m}w_j\,\psi(u_j)$, sampled at prescribed
nodes $u_1,\dots,u_m\in[0,1]$ with weights $w_1,\dots,w_m>0$. We use the
Gauss--Legendre rule, in which the nodes are the (rescaled) roots of the
Legendre polynomial of degree $m$ and the weights are chosen so that the
approximation is exact for every polynomial of degree at most $2m-1$. A function
that is smooth and well approximated by polynomials on $[0,1]$ is thus
integrated to high accuracy already for moderate $m$; the nodes cluster towards
the endpoints $0$ and $1$, where they are most needed. We say that the rule
\emph{resolves} the integrand when the $m$ sampled values $\psi(u_1),\dots,
\psi(u_m)$ capture its variation faithfully, so that the weighted sum is close
to the true integral. Resolution is assessed by increasing the number of nodes and comparing the computed solutions; smoothness alone does not supply an error estimate for this nonlinear calculation.

Writing $f_i=f(u_i)$, $g_i=g(u_i)$ and
$\eta_{ij}=e^{\beta\mathcal L_N(u_i,u_j)}$, the system above becomes the finite
nonlinear system
\begin{equation}\label{exp:disc}
f_i=\left(\frac{\sum_{j}w_j\eta_{ij}f_j+\sum_{j}w_j\eta_{ij}^{-1}g_j}{D}\right)^{2},
\quad
g_i=\left(\frac{\sum_{j}w_j\eta_{ij}^{-1}f_j+\sum_{j}w_j\eta_{ij}g_j}{D}\right)^{2},
\qquad i=1,\dots,m,
\end{equation}
with $D=\sum_{j}w_j\eta_{0j}^{-1}f_j+\sum_{j}w_j\eta_{0j}g_j$ and
$\eta_{0j}=e^{\beta\mathcal L_N(0,u_j)}$. This is a system of $2m$ equations for
the $2m$ unknowns $(f_i,g_i)$, which provides a quadrature approximation to the integral equations.
{\bf Iterative solution.}
The right-hand side of \eqref{exp:disc} defines a map
$(f,g)\mapsto\Phi_\beta(f,g)$ on positive vectors, whose fixed points are
exactly the solutions sought. The normalized equations fix the scale,
including $g(0)=1$; arbitrary common rescaling does not preserve a
solution of these equations. In the numerical iteration, fixing the
mean of the entries is a computational device. A resulting candidate
must be restored to the normalization of \eqref{uy5} and checked
against the original equations. Starting from an initial pair
$(f^{(0)},g^{(0)})$,
we form the damped iteration
\[
(f^{(\ell+1)},g^{(\ell+1)})
=(1-\theta)\,(f^{(\ell)},g^{(\ell)})
+\theta\,\Phi_\beta(f^{(\ell)},g^{(\ell)}),
\qquad \theta=\tfrac12,
\]
and continue until successive iterates differ by less than a fixed tolerance. To search for multiple solutions, the
iteration is run from several initial pairs: the symmetric choice
$f^{(0)}=g^{(0)}\equiv1$; the two polarized choices
$f^{(0)}(t)=e^{\beta t},\,g^{(0)}(t)=e^{-\beta t}$ and its interchange, which
favour the two symmetry-broken branches; and a number of random positive
initializations. Each convergent run supplies a numerical candidate,
which must be checked against \eqref{exp:disc} in its prescribed
normalization before being treated as an approximate boundary law.
Multiple initializations do not guarantee an exhaustive classification.

\begin{figure}[t]\centering
\includegraphics[width=.9\linewidth]{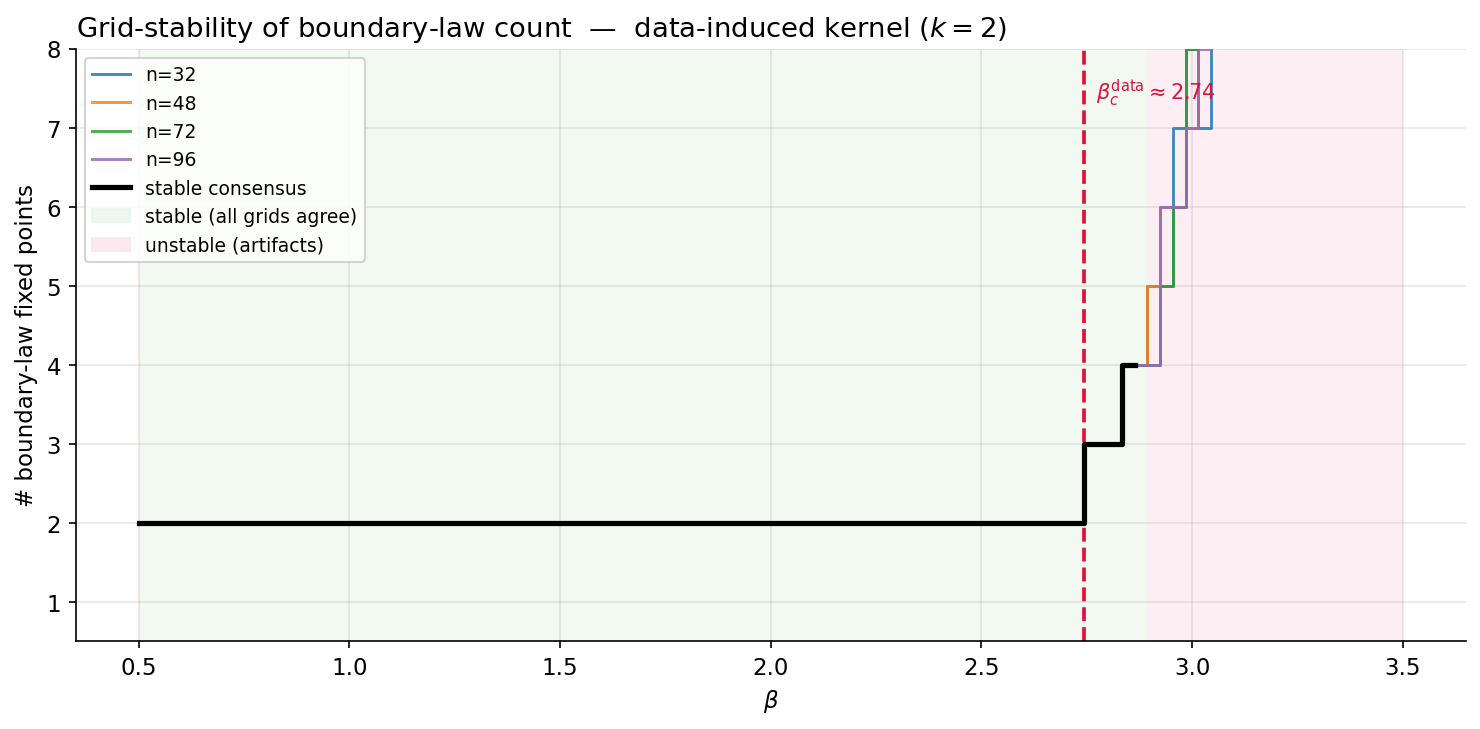}
\caption{Data-induced kernel: the grid-stable number of boundary-law solutions
reproduces the $2\to3$ transition at $\beta_c^{\mathrm{data}}\approx2.74$. There is a small region at the end of the green-shaded interval where the number of solutions transition from 3 to 4 after which (the pink-shaded interval) the numerical results for different grids disagree, rendering them inconclusive.}
\label{fig:data-transition}
\end{figure}
{\bf The observed transition.}
Let $n_{\mathrm{num}}(\beta)$ denote the number of candidate
symmetry classes retained by the numerical procedure. In the principal
resolved range, the reported count changes from two to three near
$\beta_c^{\mathrm{data}}\approx2.74$. Figure~\ref{fig:data-transition}
also reports a fourth class near the upper end of that range, followed
by disagreement between grids. The larger-$\beta$ range
where this disagreement occurs is not resolved by the computation.

The Hamiltonian remains invariant under global spin reversal
for every dataset. A non-symmetric boundary law describes a candidate
joint annealed Gibbs measure that need not preserve that symmetry;
the dataset does not remove the symmetry of the Hamiltonian itself.

{\bf Interpretation.}

We conclude with the prediction procedure, explicitly separating
model construction, measure selection, and spin inference.

\begin{enumerate}
\item Fix the dataset, predictor, loss, inverse temperature, and tree
order. Compute any observed covariates or embeddings, and form
$\mathcal L_N(t,u)$ and $\eta_{tu}=e^{\beta\mathcal L_N(t,u)}$.
The tree interaction parameters $\sigma(z)$ remain random.
\item Determine compatible normalized translation-invariant boundary
laws from \eqref{uy5}. Each law gives an annealed Gibbs measure on
interaction parameters and spins. For $k=1$, uniqueness within the
class of Theorem~\ref{thm:chain} removes the choice when its hypotheses
hold. For $k=2$, several annealed Gibbs measure may exist. 
\item If necessary, select a joint measure using a declared criterion
and tie rule; distinguish spin-reversed measures if label orientation
matters. The selected measure determines its parameter marginal and
conditional spin measures for almost every parameter configuration.
Neither uniqueness nor predictive optimality of this selection is
asserted by the boundary-law equations.
\item Assign observed binary labels to tree vertices and condition on
$\mathcal E$. For each query, use a ball containing the query and all
observations, integrate the interaction parameters, and sum all
unobserved spins as in \eqref{eq:prediction-partition}. Compute
$\widehat\varphi_\mu(x)=\mathbb E_\mu[\varphi(x)\mid\mathcal E]$.
Predict $+1$ for a positive expectation and $-1$ for a negative
expectation; use a declared tie rule when it is zero.
\item Alternatively, report predictions from several selected joint
measures. Their multiplicity does not imply that all conditional
expectations or predicted labels differ. The explicit $k=2$ example
reports a change in its algebraic branches; admissibility and
spin-reversal pairing must be checked before these are counted
as individual joint measures or distinct prediction rules.
\end{enumerate}

\section*{Data availability statements}
	The datasets generated during and/or analysed during the 
	current study are available from the corresponding author (U.A.Rozikov) on reasonable request.
	
	\section*{Conflicts of interest} The authors declare no conflicts of interest.

	\section*{ Acknowledgements}

    The authors gratefully acknowledge the University of Granada for	awarding the Visiting Scholar Grant (PPVS2024.04) to U.A. Rozikov. 
        He also expresses his sincere gratitude to the DaSCI Institute and IMAG for their kind invitation and support during his academic visit.  
	
    This publication is part of the project ``Ethical, Responsible, and General Purpose Artificial Intelligence: Applications
In Risk Scenarios" (IAFER) Exp.:TSI-100927-2023-1 funded through the creation of university-industry research
programs (ENIA Programs), aimed at the research and development of artificial intelligence, for its dissemination and
education within the framework of the Recovery, Transformation and Resilience Plan of the European Union Next
Generation EU through the Ministry of Digital Transformation and the Civil Service. This work was also partially
supported by Knowledge Generation Projects, funded by the Spanish Ministry of Science, Innovation, and Universities
of Spain under the project PID2023-150070NB-I00.

	\end{document}